%% file: arxiv.tex
\documentclass{article}

\usepackage[preprint]{neurips_2026}

\usepackage[utf8]{inputenc} 
\usepackage[T1]{fontenc}    
\usepackage{hyperref}       
\usepackage{url}            
\usepackage{booktabs}       
\usepackage{amsfonts}       
\usepackage{nicefrac}       
\usepackage{microtype}      
\usepackage{xcolor}         

\usepackage{amsmath,amsfonts,amsmath,amssymb,amsthm}

\theoremstyle{plain}
\newtheorem{theorem}{Theorem}[section]
\newtheorem{proposition}[theorem]{Proposition}
\newtheorem{lemma}[theorem]{Lemma}
\newtheorem{corollary}[theorem]{Corollary}
\theoremstyle{definition}
\newtheorem{definition}[theorem]{Definition}
\newtheorem{assumption}[theorem]{Assumption}
\theoremstyle{remark}
\newtheorem{remark}[theorem]{Remark}

\usepackage{algorithm2e}

\usepackage[capitalize,noabbrev]{cleveref}
\usepackage{authblk}
\usepackage{multirow}

\title{Optimal Guarantees for Auditing R\'enyi Differentially Private Machine Learning}

\author[1]{Benjamin D. Kim}

\author[2]{Lav R. Varshney}

\author[3]{Daniel Alabi}

\affil[1,2,3]{University of Illinois Urbana-Champaign}

\affil[1]{Massachusetts Institute of Technology}

\affil[2]{Stony Brook University}

\begin{document}

\maketitle

\begin{abstract}

We study black-box auditing for machine learning algorithms that claim R\'enyi differential privacy (RDP) guarantees. We introduce an auditing framework, based on hypothesis testing, that directly estimates R\'enyi divergence between neighboring executions using the Donsker–Varadhan (DV) variational estimator. Our analysis yields explicit and non-asymptotic confidence intervals for RDP auditing via class-restricted DV estimators, separating statistical estimation error from algorithmic privacy leakage.  We prove matching minimax lower bounds showing that, up to logarithmic factors, our sample-complexity guarantees are information-theoretically optimal, thereby establishing the \textit{first} optimal guarantees for auditing RDP via DV estimators.
Empirically, we instantiate our framework for auditing DP-SGD in a fully \textbf{black-box} setting. Across MNIST and CIFAR-10, and over a wide range of privacy regimes, our auditors produce a strong overall improvement on empirical RDP lower bounds compared to prior state-of-the-art black-box methods \textit{especially at small and moderate R\'enyi orders} where accurate auditing is most challenging.

\end{abstract}

\section{Introduction}
\label{sec:intro}

\input{sections/introduction}

\subsection{Related Work}
\label{sec:relatedwork}

\input{sections/relatedwork}

\section{Preliminaries and Background}
\label{sec:prelims}

\input{sections/prelims}
\section{Theoretical Results}
\label{sec:theory}

\input{sections/theory}

\section{Experimental Results}
\label{sec:experiments}

\input{sections/experiments}

\section{Conclusion}
\label{sec:conclusion}

\input{sections/conclusion}

\clearpage

\bibliographystyle{plain}
\bibliography{references}


\clearpage
\appendix

\newpage


\input{sections/appendix}

\end{document}

%% file: sections/introduction.tex
Differential privacy (DP) has become the de facto standard for providing rigorous privacy
guarantees in machine learning~\citep{DworkMNS06,DworkKMMN06}.
Modern private learning algorithms, most notably differentially private stochastic gradient descent (DP-SGD), 
are routinely deployed with privacy guarantees, expressed in terms of \emph{R\'enyi differential privacy}
(RDP), due to its tight composition guarantees and its central role in state-of-the-art privacy
accounting~\citep{abadi2016deep,mironov2017renyi}. Yet, as DP systems transition from theory to practice,
a fundamental question has become increasingly urgent:
\emph{How can we reliably verify, from finite empirical observations, that a learning algorithm actually
satisfies its claimed privacy guarantees?}

This question has given rise to a rapidly growing literature on \emph{privacy auditing}~\citep{jagielski2020auditing,steinke2023privacy}.
Empirical audits are used to estimate privacy leakage by distinguishing outputs of a learning algorithm on
neighboring datasets, thereby producing lower bounds on the true privacy parameters~\citep{DingWWZK18}.
Such audits play a critical role in diagnosing implementation bugs, assessing how tight the theoretical
analysis is, and building trust in deployed private learning systems~\citep{DKM19,AGMSW25}.
However, despite this progress, theoretical foundations of privacy auditing remain incomplete, particularly
for RDP.


\begin{figure}
    \begin{center}
    \includegraphics[width=0.6\linewidth]{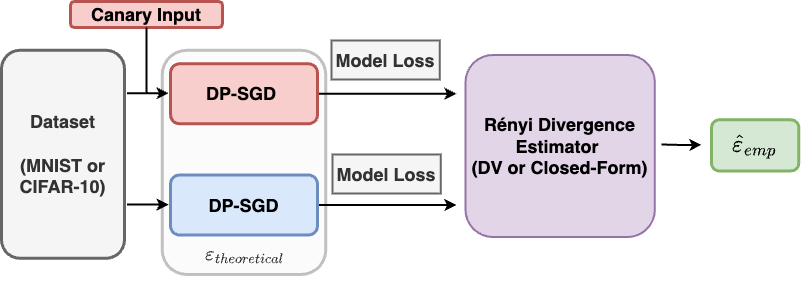}
    \caption{Auditing procedure in Algorithm~\ref{alg:bb-audit}.}
    \label{fig:overview}
    \end{center}
\end{figure}

\paragraph{The gap between empirical auditing and RDP guarantees}
Most existing auditing methods are tailored to
the original (pure and approximate) DP definitions and rely on specific attack constructions such as membership inference and
data poisoning~\citep{DingWWZK18,jagielski2020auditing}.
While effective, these approaches could lack finite-sample guarantees that cleanly separate estimation error
from algorithmic privacy leakage. 
In particular, it is currently unclear how accurately RDP parameters can be estimated from finite samples, whether
existing methods are optimal, and how auditing error should scale with model complexity.

\paragraph{Black-box privacy auditing}
Tight audits for DP-SGD exist in the white-box auditing setting where an adversary can observe and update intermediate gradients \cite{nasr2023tight}. In contrast, in this work we consider a black-box auditing setting, where an adversary can only insert an input canary and observe a (post-processing of a)
final trained model. We achieve tighter audits than previous works, \cite{jagielski2020auditing, nasr2023tight, steinke2023privacy, muthu2024nearly}. We follow
the method in \cite{muthu2024nearly} of crafting worst-case initial model parameters, since DP-SGD's privacy is unaffected by the choice of initial parameters before private training and is still considered black-box DP auditing. The approach to doing so by pretraining on a separate part of the dataset is described in Section~\ref{sec:experiments}. In short, this worst-case initialization makes training as sensitive as possible to whether the canary is present, leading to a greater R\'enyi divergence in model outputs trained with and without the canary for us.

\paragraph{Our approach: Hypothesis testing via the variational R\'enyi divergence}
In this work, we introduce a \emph{new framework for auditing R\'enyi differentially private mechanisms}
based on hypothesis
testing.
We utilize 
the Donsker-Varadhan (DV) variational representation of the R\'enyi divergence.
We view privacy auditing as a statistical estimation problem: given samples from the output distributions of a 
mechanism evaluated on neighboring datasets, the auditor seeks to estimate (or certify a lower bound on) the
corresponding R\'enyi divergence using the DV representation.
This perspective allows us to directly connect empirical auditing to
formal definitions of RDP, without relying on specific attacks or asymptotic approximations.
Our analysis yields explicit, non-asymptotic confidence intervals for class-restricted R\'enyi divergence estimates
obtained via DV-based variational objectives.
These guarantees hold uniformly over parameterized critic classes and make precise the dependence on sample size,
R\'enyi order, and function-class complexity. We can obtain statistically valid RDP audits with controlled
Type-I error, converting empirical evidence into privacy violation certificates. 




This paper develops a rigorous, non-asymptotic framework for auditing
RDP using variational (DV-based) estimators.
\textbf{Our contributions are summarized as follows:}
(i) We are the \textit{first} to directly audit R\'enyi differential privacy, establishing a state-of-the-art black-box privacy auditing method to estimate the R\'enyi divergence between neighboring dataset mechanism outputs.
Our R\'enyi divergence estimation methods between loss distributions trained with and loss distributions trained without a canary use a variational representation approach with neural networks to estimate a lower bound.
(ii) We introduce a framework, based on hypothesis testing, for auditing RDP via variational 
R\'enyi divergence (Section~\ref{sec:hypothesis-testing-renyi-dp}). We derive explicit, non-asymptotic confidence intervals for DV-based R\'enyi divergence estimators over
restricted critic classes (Theorem~\ref{thm:finite-sample-dv-renyi} and
Corollary~\ref{cor:dp-auditing-dv-renyi-app}).
We provide matching lower bounds, establishing the optimality of our auditing guarantees (Theorem~\ref{thm:dv-renyi-lower-bound}), showing that, up to logarithmic and constant factors,
the sample complexity in Theorem~\ref{thm:finite-sample-dv-renyi} is minimax optimal.
(iii) We validate our theory empirically on private machine learning systems, demonstrating accurate and robust
RDP audits in practice (Section~\ref{sec:experiments}). 
Compared to state-of-the-art black-box auditors, our work yields empirical estimates
that significantly improve over prior work.

%% file: sections/relatedwork.tex
A growing body of work studies how to empirically \emph{audit} DP machine learning
algorithms by estimating privacy leakage from observed outputs of the algorithms~\citep{ShokriSSS17,AD24,muthu2024nearly,DingWWZK18,GuanFHY23,YMMBS22,KMRS24}.
Earlier work introduced data poisoning attacks as a way to audit DP-SGD, demonstrating empirical lower bounds on
the privacy parameters can substantially exceed naive theoretical analysis~\citep{jagielski2020auditing}.
Subsequent works refined this line of research by improving tightness, reducing the number of required
training runs, or adapting to
different threat models:
~\cite{nasr2023tight} proposed tight auditing schemes that exploit knowledge of the underlying
DP mechanism to dramatically reduce sampling complexity whereas~\cite{steinke2023privacy, mahloujifar2024auditing}  showed that
meaningful DP audits can be performed using a single training run by exploiting connections between DP and
generalization. While these approaches can be effective in practice, they primarily focus on
pure or approximate DP guarantees. In contrast, our work targets
\emph{R\'enyi differential privacy} and develops a principled distribution-free statistical framework that
yields \emph{finite-sample confidence intervals} and
\emph{minimax-optimal guarantees} for auditing R\'enyi divergence directly.

Several works aim to detect DP violations or compute empirical lower bounds on privacy parameters~\citep{DingWWZK18,GuanFHY23,YMMBS22}.
\cite{DingWWZK18} extend previous formulations of DP in terms of hypothesis tests~\citep{kairouz2016extremal,kairouz2015composition} in order to detect counterexamples to DP, whereas
DP-Finder~\citep{BichselGDTV18} is used to search for DP violations as a sampling-and-optimization problem.
These tools are powerful for debugging and falsifying incorrect implementations, but they are not
designed to provide tight, sample-valid confidence bounds for \emph{correct} algorithms, nor do they address RDP
auditing or optimality guarantees.
Our approach is complementary: rather than searching for explicit counterexamples, we characterize the
\emph{best possible accuracy} of any auditor via information-theoretic lower bounds, thereby establishing
fundamental limits for RDP auditing.

There have been several previous works on variational representations of information-theoretic divergences \cite{donsker1983asymptotic, anantharam2018variational} and more recently on approximating a lower bound via neural networks with the most prominent being mutual information neural estimation (MINE) \cite{belghazi2018mine}. Recent applications \cite{esfahanizadeh2023infoshape,kim2024crypto, kim2025cryptanalysis} use these networks to perform security and privacy analyses. Our work is the first to not only perform an empirical DP audit using these divergences, but also validate the findings with theoretical confidence bounds for both the DP estimate and lower bound estimate through neural networks.

%% file: sections/prelims.tex
\subsection{R\'enyi Divergence Definition and Estimation}
\label{sec:rdp-definition}



\begin{definition}[R\'enyi divergence]
    The R\'enyi divergence of order $\alpha>1$ between two probability distributions $P$ and $Q$ defined over $\mathbb{R}$ is
\(
D_{\alpha}(P \,\|\, Q)
= \frac{1}{\alpha-1}\,
\log \,\mathbb{E}_{x \sim Q}\!\left[\left(\frac{P(x)}{Q(x)}\right)^{\alpha}\right],
\)
where all logarithms are natural; $P(x)$ and $Q(x)$ denote densities with respect to a common reference measure.
\label{def:rd}
\end{definition}
Measuring the R\'enyi divergence and other information-theoretic divergences (such as Kullback-Leibler) has presented a challenge in high-dimensional random variables until recently. The previous work of\cite{donsker1983asymptotic} presented a variational form of the KL divergence as a supremum over functions, later extended to the R\'enyi divergence by \cite{anantharam2018variational}. The idea here is to use neural networks to approximate the supremum of the class of functions, and use that approximation to calculate a lower bound of divergences \cite{belghazi2018mine, birrell2021variational}. 
\begin{definition}[Variational representations of R\'enyi divergence]
Let $P$ and $Q$ be probability measures on $(\Omega, \mathcal{M})$, $Q \ll P$, and $\alpha \in \mathbb{R},\ \alpha \neq 0, 1$. Then for any set of functions, $\Gamma$, with $\mathcal{M}_b(\Omega) \subset \Gamma \subset \mathcal{M}(\Omega)$ (where $\mathcal{M}(\Omega)$ denotes the set of all real-valued measurable functions on $\Omega$) we have
$
R_\alpha(Q\|P)
= \sup_{g \in \Gamma} \left\{ 
\frac{1}{\alpha - 1} \log\left[ \int e^{(\alpha - 1) g} \, dQ \right]
- \frac{1}{\alpha} \log\left[ \int e^{\alpha g} \, dP \right]
\right\}
$
or
$
=\sup_{g\in\Gamma}
\{\frac{1}{\alpha-1}\log \mathbb{E}_{X\sim Q}\!\left[e^{(\alpha-1)g(X)}\right] -$
$\frac{1}{\alpha}\log \mathbb{E}_{Y\sim P}\!\left[e^{\alpha g(Y)}\right]\}.
$
\end{definition}

To convert to the traditional R\'enyi divergence, the measure must be scaled by $\alpha$ so $\boldsymbol{D_\alpha(P||Q) = \alpha R_\alpha(P||Q)}$. Throughout the work we refer to the variational expression above as $R_\alpha$ and the true R\'enyi divergence as $D_\alpha$. Much like \cite{birrell2021variational}, we parameterize the variational function class with a neural network and optimize the resulting sample-based objective via stochastic gradient descent. For improved numerical stability, we use mini-batching and exponential moving average (EMA) from MINE \cite{belghazi2018mine}. EMA is a running average that mixes the current observation with a previous EMA providing stronger stability and less variance. Mini-batching allows us to optimize using Monte Carlo estimates of expectations, but can have bias and high variance due to the exponential terms in the loss functions. This issue is even more prevalent if one estimates the R\'enyi divergence with an order of $\alpha$. In attempts to remedy the bias and variance we follow MINE's approach of maintaining an EMA that averages the exponent term over several mini batches and therefore stabilizes the resulting stochastic gradients. This is especially helpful when estimating at a large order of $\alpha$ for a R\'enyi divergence. Variants using this formula in \cite{birrell2021variational} demonstrate a high mean squared error at larger $\alpha$'s, so we limit our experimental evaluation for the auditor to $\alpha \in (1, 2]$, a regime in which the estimator remains relatively accurate.

For our neural network, we choose $\Gamma_{\Theta} := \{\, T_{\theta} : \Omega \to \mathbb{R} \;:\; \theta \in \Theta \,\} \subseteq \Gamma$,
parametrized by a deep neural network with parameters $\theta \in \Theta$, the statistics network. We obtain the lower bound:
\(
R_\alpha(Q\|P) \ge R^\Theta_\alpha(Q\|P),
\)
where 
\(
R^\Theta_\alpha(Q\|P) =
\sup_{\theta\in\Theta}
\big\{
\frac{1}{\alpha-1}\log \mathbb{E}_{ Q}\!\left[e^{(\alpha-1)T_\theta}\right]
-
\frac{1}{\alpha}\log \mathbb{E}_{P}\!\left[e^{\alpha T_\theta}\right]
\big\}.
\)
The implementation for our R\'enyi divergence estimator is  \cref{alg:estimation}. The proof for the formula can be found in \cite[Section 6.2]{birrell2021variational}. Theoretical properties follow those in \cite{birrell2021variational, belghazi2018mine}, with specific analysis for DP included in later sections.

\begin{algorithm}[t]
\DontPrintSemicolon
\caption{DV-R\'enyi Divergence Estimation}
\label{alg:estimation}
\KwIn{Order $\alpha>1$, minibatch size $b$, EMA rate $\beta\in(0,1)$, step size $\eta$}
\KwOut{Lower-bound estimate $R^\Theta_\alpha(Q\|P)$}

$\theta \leftarrow$ initialize statistics network parameters\;
$m_Q \leftarrow 0$;\quad $m_P \leftarrow 0$\;

\Repeat{convergence}{
    Draw $b$ minibatch samples from $Q$:\quad $x^{(1)},\dots,x^{(b)} \sim Q$\;
    Draw $b$ minibatch samples from $P$:\quad $y^{(1)},\dots,y^{(b)} \sim P$\;

    Evaluate the lower-bound:
    \begin{align*}
        \mathcal{V}(\theta)\leftarrow
        &\frac{1}{\alpha-1}\log\!\Big(\frac{1}{b}\sum_{i=1}^b e^{(\alpha-1)T_\theta(x^{(i)})}\Big)\\
        &-\frac{1}{\alpha}\log\!\Big(\frac{1}{b}\sum_{i=1}^b e^{\alpha T_\theta(y^{(i)})}\Big)
    \end{align*}

    Evaluate bias corrected gradients (e.g., moving average):

    $\widetilde{G}(\theta)\leftarrow \widetilde{\nabla}_\theta\,\mathcal{V}(\theta)$\;

    Update the statistics network parameters (gradient ascent):
    $\theta \leftarrow \theta + \eta\,\widetilde{G}(\theta)$\;
}
\end{algorithm}

\subsection{Differential Privacy}
We now introduce DP and RDP.  We define $(\varepsilon,\delta)$-  differential privacy~\citep{DworkMNS06,DworkKMMN06}
as follows. 
\begin{definition}[Differential Privacy]
   A randomized privacy mechanism $M : \mathcal{D} \to \mathcal{R}$ satisfies $(\varepsilon, \delta)$-DP if for any adjacent $D, D' \in \mathcal{D}$ and any measurable $S \subseteq \mathcal{R}$ it holds that
\[
\Pr\!\big[M(D) \in S\big] \le e^{\varepsilon}\, \Pr\!\big[M(D') \in S\big] + \delta,
\]
where $0\le \delta \le1$.
\end{definition}

\begin{definition}[R\'enyi differential privacy]
    A randomized privacy mechanism $M : \mathcal{D} \to \mathcal{R}$
    satisfies $(\alpha, \varepsilon_\alpha)$  R\'enyi-DP \cite{mironov2017renyi}
    if for any adjacent $D, D' \in \mathcal{D}$ it holds that
\[
D_{\alpha} \big( M(D) \,\|\, M(D') \big) \leq \varepsilon_\alpha.
\]
\end{definition}
Note that as mentioned above, $\alpha \rightarrow \infty$ satisfies $\varepsilon$ differential privacy. Further, we can then convert from RDP to $(\varepsilon, \delta)$-DP using the formula provided in Proposition 3 in \cite{mironov2017renyi}. For a mechanism that satisfies $(\alpha, \varepsilon_\alpha)$ RDP, the mechanism also satisfies $(\varepsilon_\alpha + \frac{\log 1/ \delta}{\alpha -1}, \delta)$-DP for $0\le \delta \le 1$.

Since the R\'enyi Divergence is not symmetric, one takes the supremum of the divergence of all neighboring datasets mechanism outputs to see if the mechanism satisfies $(\alpha, \varepsilon_\alpha)$-RDP.

    
    \begin{algorithm}
    \caption{DP\text{-}SGD — Differentially Private Stochastic Gradient Descent \cite{abadi2016deep}}
    \label{alg:dpsgd}
    \DontPrintSemicolon
    \SetAlgoLined
    \SetKwInOut{Input}{Input}
    \SetKwInOut{Model}{Model}
    \SetKwInOut{Params}{Parameters}
    
    \Input{$x \in \mathcal{X}^n$}
    \Model{Loss function $f : \mathbb{R}^d \times \mathcal{X} \to \mathbb{R}$}
    \Params{number of iterations $\ell \ge 1$; clipping threshold $c>0$; noise multiplier $\sigma>0$; sampling probability $q \in (0,1]$; learning rate $\eta>0$}
    
    \BlankLine
    Initialize $w^{0} \in \mathbb{R}^d$.\;
    \For{$t = 1, \dots, \ell$}{
      Sample $S^{t} \subseteq [n]$ with each $i \in [n]$ included independently with probability $q$.\;
      Compute per-example gradients $g_i^{t} \gets \nabla_{w^{t-1}} f\!\left(w^{t-1}, x_i\right)$ for all $i \in S^{t}$.\;
      Clip $\hat g_i^{t} \gets \min\!\left\{1,\; \frac{c}{\lVert g_i^{t}\rVert_2}\right\} g_i^{t}$ for all $i \in S^{t}$.\;
      Sample $\xi^{t} \sim \mathcal{N}\!\left(0,\; \sigma^{2} c^{2} I\right)$.\;
      Aggregate $\tilde g^{t} \gets \xi^{t} + \sum_{i \in S^{t}} \hat g_i^{t}$.\;
      Update $w^{t} \gets w^{t-1} - \eta \cdot \tilde g^{t}$.\;
    }
    \Return{$w^{0}, w^{1}, \dots, w^{\ell}$}
    \end{algorithm}
   

\subsection{Hypothesis Testing and Statistical Estimation}

As shown in prior works, we would like to frame our auditing result as a statistical estimation problem with a bounded estimation on the privacy parameters.

To begin, we start by viewing this in terms of hypothesis testing. Our null hypothesis is we have a privacy mechanism $f : \mathcal{D} \to \mathcal{R}$ that for any adjacent $D, D' \in \mathcal{D}$, satisfies $(\alpha, \varepsilon_\alpha)$-RDP if
\(
D_{\alpha} \big( M(D) \,\|\, M(D') \big) \leq \varepsilon_\alpha.
\)
Our audit's goal is to test this hypothesis by running algorithm $M$. The output of the algorithm is $\widehat R_\alpha$. If $M$ satisfies $(\alpha, \varepsilon_\alpha)$-RDP, with probability $1- \beta$ we have $\widehat R_\alpha \le \varepsilon_\alpha$. To view this in terms of hypothesis testing, start with the null hypothesis that the output of $M$ satisfies $(\alpha, \varepsilon_{\text{null}})$-RDP. Let the hypothesis test's rejection set be $T_{\varepsilon, \alpha, \beta}$. If the null hypothesis is rejected, then we have a lower bound at $\widehat R_\alpha$. To convert between hypothesis testing and statistical estimation, consider the following lemma.
\begin{lemma}[See Lemma~\ref{lem:conversion-app}]
\label{lem:conversion}
For each $M$ and a fixed $\alpha > 1$, let $A_M \in \Omega$ be a random variable and let $P_M \in \mathbb{R}$ be a fixed number.
For each $\varepsilon,\beta>0$, let $T_{\varepsilon,\alpha,\beta}\subset \Omega$ satisfy
\begin{equation}
\forall M\qquad
\Bigl(P_M=\varepsilon \;\Longrightarrow\; \mathbb{P}\bigl[A_M \in T_{\varepsilon,\alpha,\beta}\bigr]\le \beta\Bigr).
\end{equation}
Further suppose that, if $\varepsilon_1 \le \varepsilon_2$, then
$T_{\varepsilon_1,\alpha,\beta} \supset T_{\varepsilon_2,\alpha,\beta}$.
Then, for all $M$ and all $\beta>0$,
\begin{equation}
\Pr\!\left\{\,P_M \ge
\sup\{\varepsilon>0 : A_M \in T_{\varepsilon,\alpha,\beta}\}\right\}
\ge 1-\beta.
\end{equation}
\end{lemma}

The proof can be interpreted as the following: $A_M$ is the auditing procedures output when applied to mechanism $M$ that we are auditing, $P_M$ is the true epsilon, $\varepsilon_{true}$ that the mechanism satisfies. For Eq~\ref{eqn:union}, if the null hypothesis is true, that is $M$ is $(\alpha, \varepsilon)$-RDP the probability we reject it is at most $\beta$. We estimate the true privacy parameter by taking the largest $\varepsilon$ for which the null hypothesis can still be rejected. By properties of RDP, we can make the monotonicity assumption that was used in the proof, that if we can reject $\varepsilon_2$ and $\varepsilon_1 \le \varepsilon_2$, we can also reject $\varepsilon_1$.

\subsection{Black-Box Auditing Algorithm}
In the black-box setting an adversary does not have intermediate access to a mechanism, and only can view/choose inputs and outputs. More on choosing specific model parameters and inputs can be found in our discussion. For our black-box auditing algorithm, we follow the standard algorithms shown in \cite{nasr2023tight, muthu2024nearly}. In these works, one trains the model under DP-SGD both with and without the canary input. Lastly, we measure the R\'enyi divergence of the output losses either using a variational or numerical approach. The process is depicted in Algorithm~\ref{alg:bb-audit}.

\begin{algorithm}
\caption{Black-box auditing for DP-SGD with Renyi divergence}
\label{alg:bb-audit}
\textbf{Args:} training dataset $D$, loss function $\ell$, canary input $(x',y')$, number of observations $T$\\
\textbf{Observations:} $O \gets \{\}$,\\ $O' \gets \{\}$
\For{$t = 1$ to $T$}{
    $\theta \gets \text{DP\mbox{-}SGD on dataset } D$\\
    $\theta' \gets \text{DP\mbox{-}SGD on dataset } D \cup \{(x',y')\}$\\
    $O[t] \gets \ell(\theta,\,(x',y'))$\\
    $O'[t] \gets \ell(\theta',\,(x',y'))$
}
Using each entry from 1 to $T$ in $O$, $O'$ and \texttt{Renyi} loss\;
\textbf{Estimate} for $\varepsilon_\alpha$: $\widehat R_\alpha$ = $\max\{R_\alpha$($O$, $O'),R_\alpha$ $(O'$, $O$)\}\;
\textbf{Return} $\widehat R_\alpha, \alpha$
\end{algorithm}

%% file: sections/theory.tex
In this section we present the main theoretical contributions of the paper.

\subsection{Hypothesis Testing}
\label{sec:hypothesis-testing-renyi-dp}
Our test statistic for our auditing procedure is the empirical $\varepsilon_\alpha$ estimate, $\widehat R_\alpha$ obtained by running algorithm $M$. We define our null hypothesis as for an $\varepsilon_0 > 0$, $\alpha_0 > 1$
\(
D_{\alpha_0} \big( M(D) \,\|\, M(D') \big) \leq \varepsilon_0,
\)
$\forall \varepsilon_\alpha \ge \varepsilon_0, \alpha \le \alpha_0$. To put it concretely, the null hypothesis is that the R\'enyi divergence at order $\alpha_0$ satisfies $(\alpha_0, \varepsilon_0)$-RDP. Consequently, is also satisfies $(\alpha, \varepsilon_\alpha)$-RDP, $\forall \varepsilon_\alpha \ge \varepsilon_0, \alpha \le \alpha_0$.
This means our alternative hypothesis is for an $\varepsilon_1 > 0$, $\alpha_1 > 1$
\(
D_{\alpha_0} \big( M(D) \,\|\, M(D') \big) > \varepsilon_1,
\)
The alternative hypothesis is that the R\'enyi divergence at order $\alpha_0$ does not satisfy $(\alpha_0, \varepsilon_0)$-RDP. Consequently, it also does not hold for $(\alpha, \varepsilon_\alpha)$-RDP, for all $\varepsilon_\alpha \le \varepsilon_0, \alpha \ge \alpha_0$. (See Lemma~\ref{lem:conversion}.)

\subsection{R\'enyi-DP Lower Bound Analysis}
\label{sec:rdp-analysis-lb}

In this section, we provide high-probability lower
confidence bounds for the true R\'enyi divergence
$D_\alpha(P\|Q)$ based on the empirical estimator computed by
our auditing procedure. In Section~\ref{sec:rdp-analysis-ub}, we also
discuss upper bounds.

Let $X_1,\dots,X_n \sim Q$
denote the independent loss evaluations drawn under the
canary-absent dataset, and write
\(
L(x) = \log\left(\frac{p(x)}{q(x)}\right),
\qquad
Z_i \;=\; e^{\alpha L(X_i)}, 
\widehat{Z} \;=\; \frac{1}{n} \sum_{i=1}^n Z_i,
\)
so that the empirical R\'enyi estimate returned by the auditor is
\(
\widehat{D}_\alpha(P\|Q)
\;=\;
\frac{1}{\alpha - 1}\log \widehat{Z}.
\)
Recall that
$\mu := \mathbb{E}_Q[Z_i]
= \exp\!\big( (\alpha-1) D_\alpha(P\|Q) \big)$, where $\mathbb{E}_Q[Z_i]$ is with respect to $Q$.

We now state finite-sample guarantees for estimating
$D_\alpha(P\|Q)$.  The lower bound (Theorem~\ref{thm:lower})
follows from a direct application of Markov's inequality.  
In contrast, an \emph{upper}
confidence bound requires additional structural assumptions on the
privacy-loss random variable; without such assumptions,
no nontrivial distribution-free upper bound is possible for general
nonnegative random variables.  
See Section~\ref{sec:rdp-analysis-ub} for a discussion on upper bounds.

\paragraph{Lower confidence bound.}
We first recall the standard lower-tail guarantee for the empirical
R\'enyi estimator.

\begin{theorem}[Lower bound for R\'enyi-DP estimate]
\label{thm:lower}
For any $\beta \in (0,1)$, with probability at least $1-\beta$,
\(
D_\alpha(P\|Q)
\;\ge\;
\widehat{D}_\alpha(P\|Q)
\;-\;
\frac{1}{\alpha - 1}\log\left(\frac{1}{\beta}\right).
\)
\end{theorem}

\begin{proof}
By Markov's inequality applied to $\widehat{Z}$ and the identity
$\mu = \exp\!\big((\alpha-1)D_\alpha(P\|Q)\big)$, one obtains
$\Pr[\widehat{Z} > \mu/\beta] \leq \beta$ so that with probability $\geq 1-\beta$,
$\mu\geq \beta\widehat{Z}$. Taking logarithms, we obtain
\[
D_\alpha(P\|Q) = \frac{1}{\alpha-1}\log \mu \geq \frac{1}{\alpha-1}\log\widehat{Z} + \frac{1}{\alpha-1}\log\beta 
\]
\[
=\widehat{D}_\alpha(P\|Q) - \frac{1}{\alpha-1} \log\frac{1}{\beta}.
\]
\end{proof}

We can use the result we just solved for in the following way: Say we run our audit with output $\widehat{R_\alpha}$, Following Lemma~\ref{lem:conversion}, choose a desired confidence level $0 \leq 1-\beta < 1$, and add $\frac{1}{\alpha-1}\log{\beta}$ to $\widehat{R_\alpha}$. Since if $\varepsilon_1\le \varepsilon_2$, $T_{\varepsilon_1,\alpha, \beta}\supset T_{\varepsilon_2,\alpha,\beta}$, choose $\varepsilon = \widehat R_\alpha + \frac{1}{\alpha-1}\log{\beta}$ for the lower bound. 

\subsection{DV R\'enyi Analysis}

In this section, we solve for a non-asymptotic/finite-sample confidence bound for the DV R\'enyi representation and show that the bound is tight (up to polylogarithmic factors).
As a corollary of the finite-sample bound, we obtain that the DV R\'enyi representation
estimator is consistent. That is,
the estimation error for our DV R\'enyi representation converges to 0 when using neural networks to approximate the measure.

\begin{theorem}[Consistency and finite-sample confidence interval for DV R\'enyi estimators (see Theorem~\ref{thm:finite-sample-dv-renyi-app})]
\label{thm:finite-sample-dv-renyi}
Let $\alpha\in\mathbb{R}_{> 0}\setminus\{1\}$.  
Let $P,Q$ be probability distributions on a measurable space $\Omega$.
Let $\{X_i\}_{i=1}^n \sim Q$ and $\{Y_i\}_{i=1}^n \sim P$ be independent samples.

Let $\Theta \subset \mathbb{R}^d$ be a parameter set with $\|\theta\|\le K$.
Assume the critic family $\{T_\theta:\Omega\to\mathbb{R}\}_{\theta\in\Theta}$ satisfies:
(i) (\emph{Uniform boundedness}) $\sup_{\theta,z}|T_\theta(z)| \le M$.
(ii) (\emph{Lipschitz parameterization}) For all $z\in\Omega$,
\(
|T_\theta(z)-T_{\theta'}(z)| \le L \|\theta-\theta'\|.
\)
Define the population DV R\'enyi functional
\(
V(\theta)
:=\frac{1}{\alpha-1}\log \mathbb{E}_{Q}\!\left[e^{(\alpha-1)T_\theta(X)}\right]
-\frac{1}{\alpha}\log \mathbb{E}_{P}\!\left[e^{\alpha T_\theta(Y)}\right],
\)
and its empirical estimator
\(
\widehat V_n(\theta)
:=\frac{1}{\alpha-1}\log\!\left(\frac{1}{n}\sum_{i=1}^n e^{(\alpha-1)T_\theta(X_i)}\right)
-\frac{1}{\alpha}\log\!\left(\frac{1}{n}\sum_{i=1}^n e^{\alpha T_\theta(Y_i)}\right).
\)

Define
\(
R_\alpha^{\Theta}(Q\|P):=\sup_{\theta\in\Theta} V(\theta),
\qquad
\widehat R_{\alpha,n}^{\Theta}(Q\|P):=\sup_{\theta\in\Theta} \widehat V_n(\theta).
\)

Then for any $\delta\in(0,1)$, with probability at least $1-\delta$,
\(
\big|\widehat R_{\alpha,n}^{\Theta}(Q\|P)-R_\alpha^{\Theta}(Q\|P)\big|
\;\le\;
\varepsilon_n(\delta),
\)
where
\(
\varepsilon_n(\delta)
=
C_{\alpha,M}
\left(
\sqrt{\frac{d\log\!\left(\frac{K}{\eta}\right)+\log(1/\delta)}{n}}
+
\eta
\right),$
for any $\eta>0$, and
$C_{\alpha,M}
=
O\!\left(
e^{2|\alpha|M}\max\Big\{\tfrac{1}{|\alpha-1|},\tfrac{1}{|\alpha|}\Big\}
\right).
\)

In particular, choosing $\eta = O(n^{-1/2})$ yields a valid $(1-\delta)$
finite-sample confidence interval
\(
\Big[
\widehat R_{\alpha,n}^{\Theta}(Q\|P)-\varepsilon_n(\delta),
\;
\widehat R_{\alpha,n}^{\Theta}(Q\|P)+\varepsilon_n(\delta)
\Big].
\)
Moreover,
the estimator $\widehat R_{\alpha,n}^{\Theta}(Q\|P)$ is consistent in the sense:
$\big|\widehat R_{\alpha,n}^{\Theta}(Q\|P)-R_\alpha^{\Theta}(Q\|P)\big|\rightarrow 0$ as $n\rightarrow\infty$.
\end{theorem}

Theorem~\ref{thm:finite-sample-dv-renyi} establishes that the class-restricted DV R\'enyi estimator converges at a rate
scaling as $O(d/\varepsilon^2)$ (up to logarithmic factors).
We now show that this dependence is unavoidable in general.
Specifically, we prove an information-theoretic lower bound demonstrating that,
under the same boundedness and Lipschitz assumptions on the critic class,
no estimator can uniformly estimate the population DV objective
\(\theta\mapsto V(\theta)\)
with accuracy
$\varepsilon$ and constant success probability unless the sample size scales as
$\Omega(d/\varepsilon^2)$.
This result matches Theorem~\ref{thm:finite-sample-dv-renyi}
in its leading dependence on the critic dimension and
accuracy, establishing minimax-optimality up to logarithmic factors.

\begin{theorem}[Lower bound matching Theorem~\ref{thm:finite-sample-dv-renyi} up to log factors (see Theorem~\ref{thm:dv-renyi-lower-bound-app})]
\label{thm:dv-renyi-lower-bound}
Fix \(\alpha\in\mathbb R_{>0}\setminus\{1\}\).
There exist constants
\(c,c_0,c_1>0\), depending only on \(\alpha\), such that the following holds.
For every sufficiently large \(d\), there exist a measurable space \(\Omega\), a
parameter set \(\Theta\subset\mathbb R^d\), a critic class
\(\{T_\theta:\Omega\to\mathbb R\}_{\theta\in\Theta}\), and a family of
distribution pairs
\[
    \mathcal P_d=\{(P,Q_u):u\in\mathcal U_d\}
\]
such that:
\[
    \|\theta\|_2\le 1,\quad
    |T_\theta(z)|\le 1,\quad
    |T_\theta(z)-T_{\theta'}(z)|\le c_0\|\theta-\theta'\|_2
\]
for all \(\theta,\theta'\in\Theta\) and all \(z\in\Omega\), and the following
minimax lower bound holds. Let
\[
    V_{P,Q}(\theta)
    =
    \frac{1}{\alpha-1}
    \log \mathbb E_Q\!\left[e^{(\alpha-1)T_\theta(X)}\right]
    -
    \frac{1}{\alpha}
    \log \mathbb E_P\!\left[e^{\alpha T_\theta(Y)}\right].
\]
For any estimator \(\widehat V_n:\Theta\to\mathbb R\) based on \(n\) i.i.d.
samples from \(Q\) and \(n\) i.i.d. samples from \(P\),
\[
    \inf_{\widehat V_n}
    \sup_{(P,Q)\in\mathcal P_d}
    \Pr_{P,Q}\!\left[
        \sup_{\theta\in\Theta}
        \left|\widehat V_n(\theta)-V_{P,Q}(\theta)\right|
        \ge \varepsilon
    \right]
    \ge \frac14
\]
whenever
\[
    n \le c\,\frac{d}{\varepsilon^2},
\]
for all \(0<\varepsilon\le c_1\). Consequently, any distribution-free
confidence band that controls
\(\sup_{\theta\in\Theta}|\widehat V_n(\theta)-V_{P,Q}(\theta)|\) must have
sample complexity at least \(\Omega(d/\varepsilon^2)\) in general.
\end{theorem}

Our theoretical results isolate the statistical estimation error of the class-restricted Donsker--Varadhan objective. Although Algorithm~\ref{alg:estimation} introduces optimization errors, our theorem represents an estimation-theoretic guarantee for the variational objective. Since optimization error can only reduce the value attained by the learned critic relative to the empirical supremum, the resulting lower bound becomes more conservative. The optimization error can be observed in our experiments, Section~\ref{sec:experiments}.

%% file: sections/experiments.tex


As discussed in Section~\ref{sec:rdp-definition}, in our reported audits, we convert the variational objective to the standard R\'enyi
divergence normalization. Thus, whenever Algorithm~\ref{alg:estimation} returns 
$R^\Theta_\alpha(Q\|P)$, the certified RDP lower bound is $\alpha R^\Theta_\alpha(Q\|P)$.
For simplicity, all experimental results report values in the standard RDP scale.

In this section we evaluate our auditing procedure using image datasets commonly used to benchmark differentially private machine learning models. (Additional
experimental results can be found in Section~\ref{sec:more-experiments}.)
We include results for audited privacy level $\widehat \varepsilon_{emp}$ from our neural estimation model
We collect 500 loss observations for each canary in $(O)$ and canary out $(O')$ at a privacy level for each experiment. To estimate the R\'enyi divergence which doubles as a $\widehat \varepsilon_{emp}$, we use the DV R\'enyi model.
We compare our black-box RDP results to \cite{muthu2024nearly}, performing the appropriate conversions between $\mu-$GDP \cite{dong2019gaussian}, $(\varepsilon, \delta)$-DP, and $(\alpha, \varepsilon_\alpha)$-RDP. Specifically, we start with $(\varepsilon, 10^{-5})$-DP, convert to $\mu-$GDP, and then to $(1.25, \varepsilon_\alpha)$ and $(2, \varepsilon_\alpha)$-RDP.

\subsection{Datasets and Models}
For our datasets ($\mathcal D$), we use the MNIST \cite{lecun1998mnist} and CIFAR-10 \cite{krizhevsky2009learning} datasets. MNIST is composed of $28 \times 28$ grayscale images with 60,000 training samples and 10,000 testing samples. There are 10 separate classes for digits 0--9. CIFAR-10 has 50,000 training and 10,000 testing samples on $32\times 32$ RGB images and also consists of 10 classes. 
The model that we are most interested in auditing is the convolutional neural networks (CNNs). Batch sizes are set to $n$ to ease auditing. We assume adversaries have the option to choose the canary. Previous works have experimented with both blank samples \cite{nasr2023tight} and ClipBKD \cite{jagielski2020auditing}. The blank canary just has all values of 0, and the ClipBKD canary is calculated by taking the training set, performing a principal component analysis, and using the last principal component. For our experiments we use blank canaries for our CNN models.

\subsection{Crafting Worst-Case Initial Parameters}
DP-SGD's privacy holds for both randomly initialized models as well as models with fixed parameters \cite{nasr2023tight, jagielski2020auditing}. We follow the method in \cite{muthu2024nearly} of crafting worst-case initial parameters. We pre-train the CNN models when analyzing the MNIST dataset on half of the full dataset for 5 epochs with a batch size of 32 and learning rate of $1e^{-2}$. For analyzing the CNN on CIFAR-10, we pre-train the model on the CIFAR-100 dataset for 300 epochs with batch size 128 and a learning rate of 0.1. Then, we (non-privately) fine-tune the model on half of the full dataset for 100 epochs, with a batch size of 256 and a learning rate of 0.1.

\subsection{Results}
When privately training the model we set the batch size to the dataset size and the clipping threshold to $C =1.0$. We privately train the model for 100 epochs on the other half of the dataset not used for pretraining. We privately train our models calibrating the noise $\sigma$ to satisfy $\mu-$GDP with $\mu = [0.5, 1, \sqrt 2, 2, \sqrt {10}]$, which can then be converted to $(2, \varepsilon_\alpha)$-RDP with $\varepsilon_\alpha = [0.25, 1.0, 2.0, 4.0, 10.0]$ and $(1.25, \varepsilon_\alpha)$-RDP with $\varepsilon_\alpha = [0.15625, 0.625, 1.25, 2.5, 6.25]$. Note that these conversions also satisfy $(\varepsilon, 10e^{-5})$-DP, where $\varepsilon$ is $[2.0, 4.38, 6.57, 10.0, 17.85]$. For all of our results, we include the results from the auditing method of \cite{muthu2024nearly} (yellow on graphs), which is the current state of the art for black-box auditing. We showcase experimental efficiency here with additional experiments in the appendix. 



\begin{table*}
\centering
\caption{Empirical RDP audits at $\alpha=1.25$.}
\label{tab:alpha125-1.0main}
{%
\begin{tabular}{llcc}
\toprule
Dataset & Target $\varepsilon_\alpha$ & SOTA black-box  & DV-R\'enyi \\
\midrule
\multirow{5}{*}{MNIST} & 0.15625 & $0.056 \pm 0.025$ & $\mathbf{0.115} \pm 0.080$ \\
 & 0.625 & $0.272 \pm 0.020$ & $\mathbf{0.469} \pm 0.197$ \\
 & 1.25 & $0.588 \pm 0.016$ & $\mathbf{0.995} \pm 0.289$ \\
 & 2.5 & $1.256 \pm 0.091$ & $\mathbf{2.239} \pm 0.473$ \\
 & 6.25 & $3.325 \pm 0.218$ & $\mathbf{3.884} \pm 0.327$ \\
\midrule
\multirow{5}{*}{CIFAR-10} & 0.15625 & $0.031 \pm 0.004$ & $\mathbf{0.118} \pm 0.025$ \\
 & 0.625 & $0.289 \pm 0.050$ & $\mathbf{0.551} \pm 0.056$ \\
 & 1.25 & $0.693 \pm 0.082$ & $\mathbf{1.095} \pm 0.120$ \\
 & 2.5 & $1.539 \pm 0.234$ & $\mathbf{2.248} \pm 0.344$ \\
 & 6.25 & $4.128 \pm 0.336$ & $\mathbf{4.577} \pm 0.430$ \\
\bottomrule
\end{tabular}%
}
\end{table*}

\begin{table*}
\centering
\caption{Empirical RDP lower-bound audits at $\alpha=2.0$.}
\label{tab:alpha2-1.0main}
{%
\begin{tabular}{llcc}
\toprule
Dataset & Target $\varepsilon_\alpha$ & SOTA black-box  & DV-R\'enyi \\
\midrule
\multirow{5}{*}{MNIST} 
 & 0.25 & $0.089 \pm 0.040$ & $\mathbf{0.202} \pm 0.142$ \\
 & 1.0  & $0.436 \pm 0.032$ & $\mathbf{0.867} \pm 0.339$ \\
 & 2.0  & $0.940 \pm 0.025$ & $\mathbf{1.527} \pm 0.325$ \\
 & 4.0  & $2.010 \pm 0.145$ & $\mathbf{2.127} \pm 0.250$ \\
 & 10.0 & $\mathbf{5.320} \pm 0.349$ & $2.650 \pm 0.159$ \\
\midrule
\multirow{5}{*}{CIFAR-10} 
 & 0.25 & $0.050 \pm 0.007$ & $\mathbf{0.206} \pm 0.041$ \\
 & 1.0  & $0.463 \pm 0.080$ & $\mathbf{0.953} \pm 0.213$ \\
 & 2.0  & $1.109 \pm 0.131$ & $\mathbf{1.624} \pm 0.178$ \\
 & 4.0  & $\mathbf{2.463} \pm 0.374$ & $2.250 \pm 0.200$ \\
 & 10.0 & $\mathbf{6.605} \pm 0.538$ & $2.751 \pm 0.170$ \\
\bottomrule
\end{tabular}%
}
\end{table*}

The results for $\alpha = 1.25$ are shown in Table~\ref{tab:alpha125-1.0main}, and the results for  $\alpha = 2.0$ are shown in Table~\ref{tab:alpha2-1.0main}. We show an overall improvement on the current state of the art.

\subsection{Discussion}







Our experiments show that the proposed auditing framework yields an overall improvement on black-box RDP audits than prior state-of-the-art methods across datasets, models, and privacy regimes. In particular, our DV-based estimator
recovers larger empirical R\'enyi divergence values, with the most pronounced gains at low and moderate privacy parameters, where accurate auditing is typically most difficult.

These improvements stem from directly auditing \emph{R\'enyi divergence itself}, rather than relying on indirect privacy conversions or attack-specific heuristics. By framing RDP auditing as a statistical estimation problem with explicit confidence guarantees, our approach cleanly separates estimation error from true privacy leakage, avoiding premature saturation observed in earlier methods. The use of worst-case initialization further increases statistical power by amplifying the canary's influence on training dynamics, while preserving the validity of DP-SGD's privacy guarantees.
The observed experimental trends align closely with our theoretical analysis. Smaller datasets increase the relative impact of the canary and lead to tighter audits, while larger clipping norms inject more noise and reduce detectability; both effects predicted by our finite-sample bounds. 
We note specifically that the DV-based estimator is flexible
and provides sample-valid confidence intervals, ensuring statistical soundness.
The strong empirical performance of our method is consistent with the minimax lower bounds proved in Section~\ref{sec:theory}, suggesting that further improvements are fundamentally limited by information-theoretic constraints.

%% file: sections/conclusion.tex

This work establishes a principled foundation for auditing machine learning systems that claim R\'enyi differential privacy. By formulating RDP auditing as a statistical estimation problem, we derived explicit, non-asymptotic confidence guarantees for black-box audits based on variational
R\'enyi divergence estimators. Our analysis provides both finite-sample lower confidence bounds and matching minimax lower bounds, thereby characterizing the
fundamental statistical limits of estimating the class-restricted DV R\'enyi auditing objective.
Beyond theory, our empirical results demonstrate that the proposed methods yield substantial overall improvements: our black-box audits of DP-SGD outperform previous work, especially
at small and moderate R\'enyi orders, 
across a range of datasets and model architectures. 
Several directions remain open:
An important avenue for future work is extending optimal R\'enyi auditing guarantees to interactive and distributed settings.
Another is exploring alternative variational formulas~\cite{birrell2023functionspace} that exhibit lower
variance for larger values of $\alpha > 1$.

%% file: sections/appendix.tex
\section{Notions of DP and their conversions}
In our experiments, we convert our results from traditional $(\varepsilon, \delta)$-DP to Gaussian DP ($\mu-$GDP) \cite{dong2019gaussian} to $(\alpha, \varepsilon_\alpha)$-RDP. We convert with Corollary 2.13 and B6 in \cite{dong2019gaussian}.

\begin{corollary}
    (2.13): A mechanism is $\mu$-GDP if and only if it is $(\varepsilon,\delta(\varepsilon))$-DP for all $\varepsilon \ge 0$, where:
\begin{equation}
\delta(\varepsilon)
= \Phi\!\left(-\frac{\varepsilon}{\mu} + \frac{\mu}{2}\right)
- e^{\varepsilon}\,\Phi\!\left(-\frac{\varepsilon}{\mu} - \frac{\mu}{2}\right).
\end{equation}
\end{corollary}

\begin{corollary}
    (B6): If a mechanism is $\mu$-GDP, then it is ($\alpha, \frac{1}{2}\mu^2\alpha $)-RDP for any $\alpha >1$.
\end{corollary}

\section{Additional Theoretical Results and Omitted Proofs (from Main Body)}


This appendix provides supporting theoretical results that complement and strengthen the main guarantees in the main body.
While the main body of the paper focuses on just stating lower confidence bounds and minimax-optimal estimation of R\'enyi divergence via
variational (DV-based) methods, we provide proofs we could not include
in the main body (due to space restrictions).

In this appendix, our results clarify three important aspects:

\begin{enumerate}
    \item \textbf{What can and cannot be estimated from finite samples.}
    We show that lower confidence bounds on R\'enyi divergence are available under no assumptions, whereas
    meaningful upper confidence bounds are impossible without additional structural constraints on the privacy-loss random variable.
    This delineates a sharp boundary between what empirical privacy auditing can guarantee in a distribution-free manner and
    what requires stronger modeling assumptions.

    \item \textbf{How bounded privacy loss enables two-sided confidence intervals.}
    By imposing a bounded privacy-loss assumption (exact for pure DP and a reasonable high-probability surrogate for many RDP mechanisms)
    we derive finite-sample upper and lower confidence bounds using classical concentration inequalities.
    These results justify when empirical audits can certify both privacy violations and near-tight compliance.

    \item \textbf{Why the DV-based estimator is statistically sound for auditing.}
    We provide a full non-asymptotic analysis of the DV variational estimator over restricted critic classes,
    showing uniform convergence, explicit rates, and consistency.
    This formally validates the use of neural divergence estimators as statistically principled auditing tools,
    rather than heuristic approximations.
\end{enumerate}

Our results complete the theoretical foundation of our auditing framework by characterizing both its guarantees
and its fundamental limitations.

\subsection{R\'enyi-DP Upper Bound Analysis}
\label{sec:rdp-analysis-ub}

We adopt a natural bounded
privacy-loss assumption, which holds exactly for $(\varepsilon,0)$-DP
mechanisms and serves as an effective high-probability surrogate
for many RDP mechanisms in practice.

\paragraph{Upper confidence bound under bounded privacy loss.}
In contrast to the lower bound, an upper bound on
$D_\alpha(P\|Q)$ from samples cannot be obtained without further
assumptions: there exist distributions with arbitrarily large means
that nevertheless produce finitely many small observations with high
probability.  We therefore impose the following assumption,
natural for many DP mechanisms:

\begin{assumption}[Bounded privacy loss]
\label{assump:bounded-L}
There exists $B \in \mathbb{R}$ such that 
$L(x) \le B$ for all $x$ in the support of $Q$.
Equivalently,
$\frac{p(x)}{q(x)} \le e^B$ for all $x$.
Under this assumption,
\[
0 \;\le\; Z_i = e^{\alpha L(X_i)} \;\le\; e^{\alpha B} \;=:\; M
\qquad\text{for all $i$.}
\]
\end{assumption}

This is exact for $(\varepsilon,0)$-DP mechanisms, where $B=\varepsilon$,
and serves as a practical high-probability truncation model for
Gaussian or other mechanisms.

\begin{theorem}[Upper bound for R\'enyi-DP estimate under bounded privacy loss]
\label{thm:upper}
Suppose Assumption~\ref{assump:bounded-L} holds.  Then for any
$\beta \in (0,1)$, with probability at least $1-\beta$,
\[
D_\alpha(P\|Q)
\;\le\;
\frac{1}{\alpha - 1}
\log\left(
\widehat{Z}
\;+\;
e^{\alpha B}
\sqrt{\frac{\log(2/\beta)}{2n}}
\right).
\]
Equivalently, defining
\[
\widehat{\varepsilon}_{\alpha}^{\mathrm{upper}}
:=
\frac{1}{\alpha - 1}
\log\left(
\widehat{Z}
\;+\;
e^{\alpha B}
\sqrt{\tfrac{\log(2/\beta)}{2n}}
\right),
\]
we have
$\Pr\!\left[\,D_\alpha(P\|Q)
\le \widehat{\varepsilon}_{\alpha}^{\mathrm{upper}}\,\right]
\ge 1-\beta$.
\end{theorem}

\begin{proof}
Under Assumption~\ref{assump:bounded-L},
$Z_i \in [0,M]$ almost surely.  Hoeffding's inequality yields
\[
\Pr\!\left[
\bigl|\widehat{Z} - \mu \bigr|
\ge t
\right]
\;\le\;
2\exp\left( - \frac{2n t^2}{M^2} \right)
\qquad\text{for all $t>0$}.
\]
Setting 
$t = M \sqrt{\log(2/\beta)/(2n)}$ gives a deviation probability at most
$\beta$, hence with probability at least $1-\beta$,
\[
\mu \;\le\; \widehat{Z} + 
M \sqrt{\tfrac{\log(2/\beta)}{2n}}
=
\widehat{Z} + e^{\alpha B}
\sqrt{\tfrac{\log(2/\beta)}{2n}}.
\]
Since
$\mu=\exp\!\big((\alpha-1)D_\alpha(P\|Q)\big)$,
taking logarithms and dividing by $(\alpha-1)$ gives the stated
upper bound.
\end{proof}

Unlike lower confidence bounds, which follow directly from Markov-type arguments and hold for arbitrary nonnegative random variables,
upper confidence bounds on R\'enyi divergence cannot be obtained in a fully distribution-free manner.
So, a random variable may have arbitrarily large expectation while producing small empirical averages with high probability,
making any finite-sample upper bound vacuous without further structure.

For privacy auditing, this means that while violations of claimed RDP guarantees can always be certified from data,
certifying near-tight compliance necessarily requires assumptions about the privacy-loss distribution.
The bounded privacy-loss assumption introduced below captures exactly the setting of pure differential privacy
and serves as a realistic approximation for many RDP mechanisms used in practice.

\subsection{Tighter lower bound with Hoeffding's Inequality}
We can prove a tighter lower bound for the R\'enyi divergence using similar assumptions to Theorem~\ref{thm:upper}.

\begin{theorem}[Lower confidence bound under bounded privacy loss (Hoeffding)]
\label{thm:hoeffding-lower}
Fix $\alpha>1$. Assume Assumption~\ref{assump:bounded-L} holds. Then for any $\beta\in(0,1)$, with probability at least $1-\beta$,
\[
D_\alpha(P\|Q)\ \ge\ \frac{1}{\alpha-1}\log\left(\max\left\{1,\ \widehat Z - M\sqrt{\frac{\log(1/\beta)}{2n}}\right\}\right).
\]
\end{theorem}

\begin{proof}
Under the bounded privacy-loss assumption, we have $0\le Z_i\le M$ almost surely.
Hence $\widehat Z = \frac1n\sum_{i=1}^n Z_i$ is an average of i.i.d.\ bounded random variables.
By Hoeffding's inequality (one-sided lower tail), for all $t>0$,
\begin{equation}
\label{eq:hoeffding-one-sided}
\Pr\left(\widehat Z - \mu \le -t\right)\ \le\ \exp\left(-\frac{2nt^2}{M^2}\right).
\end{equation}
Set
\[
t := M\sqrt{\frac{\log(1/\beta)}{2n}}.
\]
Then $\exp\left(-\frac{2nt^2}{M^2}\right)=\beta$, and \eqref{eq:hoeffding-one-sided} becomes
\[
\Pr\left(\widehat Z - \mu \le -t\right)\le \beta,
\]
which is equivalent to: with probability at least $1-\beta$,
\begin{equation}
\label{eq:mu-lb}
\mu \ \ge\ \widehat Z - t.
\end{equation}
Also, for $\alpha>1$, R\'enyi divergence is nonnegative, hence $\mu\ge 1$.
Combining this deterministic fact with \eqref{eq:mu-lb}, we obtain on the same event
\[
\mu \ \ge\ \max\{1,\widehat Z - t\}.
\]
We know that $D_\alpha(P\|Q)=\frac1{\alpha-1}\log \mu$, which yields
\[
D_\alpha(P\|Q)
= \frac{1}{\alpha-1}\log \mu
\ \ge\
\frac{1}{\alpha-1}\log\!\left(\max\{1,\widehat Z - t\}\right).
\]
Substituting $t = M\sqrt{\frac{\log(1/\beta)}{2n}}$ and $M=e^{\alpha B}$ completes the proof.
\end{proof}

\subsection{Finite-sample Confidence Intervals for DV R\'enyi estimators}

\paragraph{Why Uniform Convergence of the DV Estimator Matters}
The DV variational representation is central to our auditing methodology,
as it allows R\'enyi divergence to be estimated via optimization over a parameterized function class.
However, for auditing purposes, pointwise convergence for a fixed critic is insufficient:
the estimator must converge \emph{uniformly} over the entire class to justify taking a supremum.

Theorem~\ref{thm:finite-sample-dv-renyi-app} establishes this uniform convergence under mild boundedness and Lipschitz assumptions,
yielding explicit finite-sample confidence intervals and consistency guarantees.
This result bridges modern neural divergence estimation techniques with classical statistical learning theory,
and ensures that the resulting privacy audits admit rigorous, sample-valid interpretation.

\begin{theorem}[Consistency and finite-sample confidence interval for DV R\'enyi estimators]
\label{thm:finite-sample-dv-renyi-app}
Let $\alpha\in\mathbb{R}_{> 0}\setminus\{1\}$.  
Let $P,Q$ be probability distributions on a measurable space $\Omega$.
Let $\{X_i\}_{i=1}^n \sim Q$ and $\{Y_i\}_{i=1}^n \sim P$ be independent samples.

Let $\Theta \subset \mathbb{R}^d$ be a parameter set with $\|\theta\|\le K$.
Assume the critic family $\{T_\theta:\Omega\to\mathbb{R}\}_{\theta\in\Theta}$ satisfies:
\begin{enumerate}
\item (\textbf{Uniform boundedness}) $\sup_{\theta,z}|T_\theta(z)| \le M$.
\item (\textbf{Lipschitz parameterization}) For all $z\in\Omega$,
\[
|T_\theta(z)-T_{\theta'}(z)| \le L \|\theta-\theta'\| .
\]
\end{enumerate}

Define the population DV R\'enyi functional
\[
V(\theta)
:=\frac{1}{\alpha-1}\log \mathbb{E}_{Q}\!\left[e^{(\alpha-1)T_\theta(X)}\right]
-\frac{1}{\alpha}\log \mathbb{E}_{P}\!\left[e^{\alpha T_\theta(Y)}\right],
\]
and its empirical estimator
\[
\widehat V_n(\theta)
:=\frac{1}{\alpha-1}\log\!\left(\frac{1}{n}\sum_{i=1}^n e^{(\alpha-1)T_\theta(X_i)}\right)
-\frac{1}{\alpha}\log\!\left(\frac{1}{n}\sum_{i=1}^n e^{\alpha T_\theta(Y_i)}\right).
\]

Define
\[
R_\alpha^{\Theta}(Q\|P):=\sup_{\theta\in\Theta} V(\theta),
\qquad
\widehat R_{\alpha,n}^{\Theta}(Q\|P):=\sup_{\theta\in\Theta} \widehat V_n(\theta).
\]

Then for any $\delta\in(0,1)$, with probability at least $1-\delta$,
\[
\big|\widehat R_{\alpha,n}^{\Theta}(Q\|P)-R_\alpha^{\Theta}(Q\|P)\big|
\;\le\;
\varepsilon_n(\delta),
\]
where
\[
\varepsilon_n(\delta)
=
C_{\alpha,M}
\left(
\sqrt{\frac{d\log\!\left(\frac{K}{\eta}\right)+\log(1/\delta)}{n}}
+
\eta
\right),
\]
for any $\eta>0$, and
\[
C_{\alpha,M}
=
O\!\left(
e^{2|\alpha|M}\max\Big\{\tfrac{1}{|\alpha-1|},\tfrac{1}{|\alpha|}\Big\}
\right).
\]

In particular, choosing $\eta = O(n^{-1/2})$ yields a valid $(1-\delta)$
finite-sample confidence interval
\[
\Big[
\widehat R_{\alpha,n}^{\Theta}(Q\|P)-\varepsilon_n(\delta),
\;
\widehat R_{\alpha,n}^{\Theta}(Q\|P)+\varepsilon_n(\delta)
\Big].
\]
In addition,
the estimator $\widehat R_{\alpha,n}^{\Theta}(Q\|P)$ is consistent in the following sense:
$\big|\widehat R_{\alpha,n}^{\Theta}(Q\|P)-R_\alpha^{\Theta}(Q\|P)\big|\rightarrow 0$ as $n\rightarrow\infty$.
\end{theorem}

\begin{proof}
The proof for the finite-sample bound proceeds in four steps.

\paragraph{Step 1: Concentration for fixed $\theta$.}
Define
\[
A_\theta := \mathbb{E}_{Q}[e^{(\alpha-1)T_\theta(X)}],
\qquad
\widehat A_\theta := \frac{1}{n}\sum_{i=1}^n e^{(\alpha-1)T_\theta(X_i)}.
\]
Since $|T_\theta|\le M$,
\[
e^{-|\alpha-1|M} \le e^{(\alpha-1)T_\theta(X)} \le e^{|\alpha-1|M}.
\]
Thus $\widehat A_\theta$ is an average of i.i.d.\ bounded random variables.
Hoeffding's inequality implies that for all $t>0$,
\[
\mathbb{P}\left(|\widehat A_\theta-A_\theta|\ge t\right)
\le
2\exp\left(-\frac{n t^2}{2e^{2|\alpha-1|M}}\right).
\]

Similarly, define
\[
B_\theta := \mathbb{E}_{P}[e^{\alpha T_\theta(Y)}],
\qquad
\widehat B_\theta := \frac{1}{n}\sum_{i=1}^n e^{\alpha T_\theta(Y_i)},
\]
and obtain
\[
\mathbb{P}\left(|\widehat B_\theta-B_\theta|\ge t\right)
\le
2\exp\left(-\frac{n t^2}{2e^{2|\alpha|M}}\right).
\]

\paragraph{Step 2: From moment errors to log errors.}
Since $A_\theta,\widehat A_\theta \ge e^{-|\alpha-1|M}$,
the logarithm is Lipschitz on this interval with constant $e^{|\alpha-1|M}$.
Thus
\[
|\log \widehat A_\theta-\log A_\theta|
\le
e^{|\alpha-1|M}|\widehat A_\theta-A_\theta|.
\]
An analogous bound holds for $\log \widehat B_\theta$.

Combining both terms,
\[
|\widehat V_n(\theta)-V(\theta)|
\le
\frac{e^{|\alpha-1|M}}{|\alpha-1|}|\widehat A_\theta-A_\theta|
+
\frac{e^{|\alpha|M}}{|\alpha|}|\widehat B_\theta-B_\theta|.
\]
Therefore there exists a constant $c>0$ such that
\[
\mathbb{P}\left(|\widehat V_n(\theta)-V(\theta)|\ge u\right)
\le
4\exp\left(-c\,n\,u^2 e^{-4|\alpha|M}\right).
\]

\paragraph{Step 3: Uniformization over $\Theta$ via covering numbers.}
Let $\mathcal{N}_\eta$ be an $\eta$-net of $\Theta$ in Euclidean norm.
Since $\Theta\subseteq B(0,K)\subset\mathbb{R}^d$,
\[
|\mathcal{N}_\eta| \le \left(\frac{3K}{\eta}\right)^d.
\]
Applying a union bound,
\[
\mathbb{P}\left(
\max_{\vartheta\in\mathcal{N}_\eta}
|\widehat V_n(\vartheta)-V(\vartheta)|
\ge u
\right)
\le
|\mathcal{N}_\eta|\cdot
4\exp\left(-c\,n\,u^2 e^{-4|\alpha|M}\right).
\]
Choosing
\[
u
=
C_{\alpha,M}
\sqrt{\frac{d\log(K/\eta)+\log(1/\delta)}{n}}
\]
ensures the above probability is at most $\delta/2$.

\paragraph{Step 4: Extension from the net to all parameters.}
The Lipschitz assumption implies that both $V(\theta)$ and $\widehat V_n(\theta)$
are Lipschitz in $\theta$ with constant
$O(|\alpha|Le^{2|\alpha|M})$.
Hence for any $\theta\in\Theta$ and nearest net point $\vartheta\in\mathcal{N}_\eta$,
\[
|\widehat V_n(\theta)-V(\theta)|
\le
|\widehat V_n(\vartheta)-V(\vartheta)| + O(\eta).
\]
Taking the supremum over $\theta$ yields
\[
\sup_{\theta\in\Theta}|\widehat V_n(\theta)-V(\theta)|
\le
\max_{\vartheta\in\mathcal{N}_\eta}|\widehat V_n(\vartheta)-V(\vartheta)| + O(\eta).
\]

\paragraph{Step 5: Confidence interval for the supremum.}
Finally, for any functions $f,g$,
\[
|\sup f - \sup g| \le \sup |f-g|.
\]
Applying this inequality completes the proof of the finite-sample bound.

The consistency statement follows from the finite-sample bound.

\end{proof}

\begin{corollary}[DP-auditing via finite-sample DV R\'enyi certificates]
\label{cor:dp-auditing-dv-renyi-app}
Fix $\alpha\in(1,\infty)$ and let $\delta_{\mathrm{CI}}\in(0,1)$.
Consider neighboring datasets $D,D'$ and let $P:=\mathsf{M}(D)$ and $Q:=\mathsf{M}(D')$
denote the output distributions of a (possibly randomized) mechanism $\mathsf{M}$.
Let $\Theta\subset\mathbb{R}^d$ and $\{T_\theta\}_{\theta\in\Theta}$ satisfy the assumptions of
Theorem~\ref{thm:finite-sample-dv-renyi} (boundedness and Lipschitz parameterization).

Suppose we can draw independent samples
$\{Y_i\}_{i=1}^n\stackrel{iid}{\sim}P$ and $\{X_i\}_{i=1}^n\stackrel{iid}{\sim}Q$,
and compute the empirical DV R\'enyi estimator
\[
\widehat R_{\alpha,n}^{\Theta}(Q\|P):=\sup_{\theta\in\Theta}\widehat V_n(\theta)
\quad\text{as defined in Theorem~\ref{thm:finite-sample-dv-renyi}.}
\]
Let $\varepsilon_n(\delta_{\mathrm{CI}})$ be the corresponding confidence radius from
Theorem~\ref{thm:finite-sample-dv-renyi}.
Define the one-sided lower confidence bound (LCB)
\[
\mathrm{LCB}_{\alpha,n}
\;:=\;
\widehat R_{\alpha,n}^{\Theta}(Q\|P)-\varepsilon_n(\delta_{\mathrm{CI}}).
\]

Then with probability at least $1-\delta_{\mathrm{CI}}$ over the samples,
\[
R_\alpha^{\Theta}(Q\|P)\;\ge\;\mathrm{LCB}_{\alpha,n}.
\]
Consequently, for any claimed R\'enyi-DP level $\rho\ge 0$ (at order $\alpha$),
the following \emph{auditing test} is valid:
\[
\text{Reject the null }H_0:\ R_\alpha^{\Theta}(Q\|P)\le \rho
\quad\text{if }\ \mathrm{LCB}_{\alpha,n}>\rho .
\]
This test has Type-I error at most $\delta_{\mathrm{CI}}$:
\[
\sup_{(P,Q):\,R_\alpha^{\Theta}(Q\|P)\le \rho}\ 
\mathbb{P}\big(\mathrm{LCB}_{\alpha,n}>\rho\big)\ \le\ \delta_{\mathrm{CI}}.
\]

Moreover, this yields a conservative $(\varepsilon,\delta)$-DP \emph{violation certificate}:
for any target $\delta\in(0,1)$, define
\[
\varepsilon_{\mathrm{LCB}}(\delta)
:=
\frac{\alpha-1}{\alpha}\,\mathrm{LCB}_{\alpha,n}
+\frac{1}{\alpha}\log\!\Big(\frac{1}{\delta}\Big).
\]
Then, with probability at least $1-\delta_{\mathrm{CI}}$,
the mechanism $\mathsf{M}$ cannot satisfy $(\varepsilon,\delta)$-DP for any
$\varepsilon<\varepsilon_{\mathrm{LCB}}(\delta)$; equivalently,
\[
\text{if }\ \varepsilon<\varepsilon_{\mathrm{LCB}}(\delta),\ \text{ then }\ \mathsf{M}\ \text{violates }(\varepsilon,\delta)\text{-DP}.
\]
\end{corollary}

\begin{proof}
The first claim is immediate from Theorem~\ref{thm:finite-sample-dv-renyi}:
with probability at least $1-\delta_{\mathrm{CI}}$,
\[
\big|\widehat R_{\alpha,n}^{\Theta}(Q\|P)-R_\alpha^{\Theta}(Q\|P)\big|
\le \varepsilon_n(\delta_{\mathrm{CI}}),
\]
which rearranges to
$R_\alpha^{\Theta}(Q\|P)\ge \widehat R_{\alpha,n}^{\Theta}(Q\|P)-\varepsilon_n(\delta_{\mathrm{CI}})
=\mathrm{LCB}_{\alpha,n}$.

For the hypothesis test, under the null $H_0$ we have
$R_\alpha^{\Theta}(Q\|P)\le \rho$.
On the event $\{R_\alpha^{\Theta}(Q\|P)\ge \mathrm{LCB}_{\alpha,n}\}$ (which holds with
probability at least $1-\delta_{\mathrm{CI}}$),
the inequality $\mathrm{LCB}_{\alpha,n}>\rho$ cannot occur.
Hence
\[
\mathbb{P}_{H_0}\big(\mathrm{LCB}_{\alpha,n}>\rho\big)
\le
\mathbb{P}\big(R_\alpha^{\Theta}(Q\|P)< \mathrm{LCB}_{\alpha,n}\big)
\le \delta_{\mathrm{CI}}.
\]

For the $(\varepsilon,\delta)$-DP violation certificate, recall the standard implication:
if a mechanism is $\rho$-RDP at order $\alpha>1$, i.e.
$R_\alpha(Q\|P)\le \rho$ for all neighboring pairs, then it is also
$(\varepsilon,\delta)$-DP with
\[
\varepsilon = \rho + \frac{\log(1/\delta)}{\alpha-1}.
\]
Equivalently (contrapositive), if for some neighboring $(P,Q)$ we have
\[
R_\alpha(Q\|P) > \varepsilon - \frac{\log(1/\delta)}{\alpha-1},
\]
then the mechanism violates $(\varepsilon,\delta)$-DP.
Applying this contrapositive with the class-restricted divergence and the lower bound
$R_\alpha^{\Theta}(Q\|P)\ge \mathrm{LCB}_{\alpha,n}$ shows that on the same high-probability
event, whenever
\[
\mathrm{LCB}_{\alpha,n} > \varepsilon - \frac{\log(1/\delta)}{\alpha-1},
\]
the mechanism cannot satisfy $(\varepsilon,\delta)$-DP.
Rearranging gives the stated threshold
\[
\varepsilon < \mathrm{LCB}_{\alpha,n} + \frac{\log(1/\delta)}{\alpha-1}.
\]
\end{proof}

\begin{theorem}[Lower bound matching Theorem~\ref{thm:finite-sample-dv-renyi-app} up to log factors]
\label{thm:dv-renyi-lower-bound-app}
Fix \(\alpha\in\mathbb R_{>0}\setminus\{1\}\). There exist constants
\(c,c_0,c_1>0\), depending only on \(\alpha\), such that the following holds.
For every sufficiently large \(d\), there exist a measurable space \(\Omega\), a
parameter set \(\Theta\subset\mathbb R^d\), a critic class
\(\{T_\theta:\Omega\to\mathbb R\}_{\theta\in\Theta}\), and a family of
distribution pairs
\[
    \mathcal P_d=\{(P,Q_u):u\in\mathcal U_d\}
\]
such that:
\[
    \|\theta\|_2\le 1,\qquad
    |T_\theta(z)|\le 1,\qquad
    |T_\theta(z)-T_{\theta'}(z)|\le c_0\|\theta-\theta'\|_2
\]
for all \(\theta,\theta'\in\Theta\) and all \(z\in\Omega\), and the following
minimax lower bound holds. Let
\[
    V_{P,Q}(\theta)
    =
    \frac{1}{\alpha-1}
    \log \mathbb E_Q\!\left[e^{(\alpha-1)T_\theta(X)}\right]
    -
    \frac{1}{\alpha}
    \log \mathbb E_P\!\left[e^{\alpha T_\theta(Y)}\right].
\]
For any estimator \(\widehat V_n:\Theta\to\mathbb R\) based on \(n\) i.i.d.
samples from \(Q\) and \(n\) i.i.d. samples from \(P\),
\[
    \inf_{\widehat V_n}
    \sup_{(P,Q)\in\mathcal P_d}
    \Pr_{P,Q}\!\left[
        \sup_{\theta\in\Theta}
        \left|\widehat V_n(\theta)-V_{P,Q}(\theta)\right|
        \ge \varepsilon
    \right]
    \ge \frac14
\]
whenever
\[
    n \le c\,\frac{d}{\varepsilon^2},
\]
for all \(0<\varepsilon\le c_1\). Consequently, any distribution-free
confidence band that controls
\(\sup_{\theta\in\Theta}|\widehat V_n(\theta)-V_{P,Q}(\theta)|\) must have
sample complexity at least \(\Omega(d/\varepsilon^2)\) in general.
\end{theorem}

\begin{proof}
We prove the result by a Fano reduction from multi-way hypothesis testing~\citep{CoverT2006}. The
construction is chosen so that the population DV objective
\(\theta\mapsto V_{P,Q}(\theta)\) encodes a hidden \(d\)-bit vector, while the
KL divergence between the corresponding sample distributions remains of order
\(n\varepsilon^2\)~\cite{Alabi26}. This yields the desired \(d/\varepsilon^2\) lower bound.

\paragraph{Step 1: A balanced packing.}
Let \(d\) be even. 
(For odd \(d\), apply the construction in dimension \(d-1\) and embed it into
\(\mathbb R^d\) by adding one unused coordinate; this changes constants only.)
By the Gilbert--Varshamov bound restricted to the middle
slice of the hypercube~\citep{bok:MW}, there exists a set
\[
    \mathcal U_d\subseteq \{-1,+1\}^d
\]
such that every \(u\in\mathcal U_d\) is balanced,
\[
    \sum_{i=1}^d u_i=0,
\]
and
\[
    |\mathcal U_d|\ge \exp(c_{\mathrm{VG}}d),
\]
for a universal constant \(c_{\mathrm{VG}}>0\), and for every distinct
\(u,v\in\mathcal U_d\),
\[
    \frac d4 \le d_H(u,v)\le \frac{3d}{4}.
\]
Equivalently, writing
\[
    \rho(u,v):=\frac1d\sum_{i=1}^d u_i v_i,
\]
we have
\[
    \rho(u,u)=1,
    \qquad
    \rho(u,v)\le \frac12
    \quad\text{for all }u\ne v.
\]
Indeed, since
\[
    \rho(u,v)=1-\frac{2d_H(u,v)}{d},
\]
the lower bound \(d_H(u,v)\ge d/4\) gives \(\rho(u,v)\le 1/2\).

\paragraph{Step 2: Parameter set and critic class.}
Let
\[
    \Theta:=\left\{\theta^u:=\frac{u}{\sqrt d}:u\in\mathcal U_d\right\}
    \subset \mathbb R^d.
\]
Then \(\|\theta^u\|_2=1\) for every \(u\in\mathcal U_d\). Let the sample space be
\[
    \Omega:=\{1,2,\ldots,d\}.
\]
Fix a constant \(\tau\in(0,1]\), to be chosen sufficiently small depending only
on \(\alpha\). For \(\theta^u\in\Theta\), define
\[
    T_{\theta^u}(i):=\tau u_i,\qquad i\in\Omega.
\]
Then \(|T_{\theta^u}(i)|\le \tau\le 1\).

We next verify Lipschitzness. For distinct \(u,v\in\mathcal U_d\),
\[
    |T_{\theta^u}(i)-T_{\theta^v}(i)|
    =
    \tau |u_i-v_i|
    \le 2\tau.
\]
On the other hand,
\[
    \|\theta^u-\theta^v\|_2
    =
    \frac{1}{\sqrt d}\|u-v\|_2
    =
    \frac{2}{\sqrt d}\sqrt{d_H(u,v)}
    \ge 1,
\]
because \(d_H(u,v)\ge d/4\). Therefore
\[
    |T_{\theta^u}(i)-T_{\theta^v}(i)|
    \le 2\tau \|\theta^u-\theta^v\|_2.
\]
Thus the critic class is Lipschitz on \(\Theta\) with constant \(L=2\tau\), and
has \(K=1\), \(M=\tau\), all absolute constants. Since the theorem only needs a
critic class indexed by \(\Theta\), no extension outside \(\Theta\) is required.
If one wants \(\Theta\) to be a full subset with an extension to all of
\(\mathbb R^d\), one may apply the McShane extension theorem~\citep{McSHANE1934ExtensionOR,Whitney34} pointwise in
\(i\), preserving the same Lipschitz constant and boundedness after clipping.

\paragraph{Step 3: A family of distributions.}
Let \(P\) be the uniform distribution on \(\Omega\):
\[
    P(i)=\frac1d,\qquad i=1,\ldots,d.
\]
For each hidden vector \(u\in\mathcal U_d\), define \(Q_u\) by
\[
    Q_u(i):=\frac{1+\delta u_i}{d},
    \qquad i=1,\ldots,d,
\]
where \(0<\delta\le 1/2\). Since \(u\) is balanced,
\[
    \sum_{i=1}^d Q_u(i)
    =
    \frac1d\sum_{i=1}^d (1+\delta u_i)
    =
    1+\frac{\delta}{d}\sum_{i=1}^d u_i
    =
    1.
\]
Moreover, \(Q_u(i)>0\) for all \(i\), so \(Q_u\) is a valid probability
distribution.

The statistical experiment associated with \(u\) consists of \(n\) i.i.d.
samples from \(Q_u\) and \(n\) i.i.d. samples from the fixed distribution \(P\).
The samples from \(P\) carry no information about \(u\), but they are included
to match the DV estimation setting.

\paragraph{Step 4: Compute the DV functional.}
Fix \(u,v\in\mathcal U_d\). We evaluate the DV objective under the pair
\((P,Q_u)\) at the critic \(\theta^v\). Write
\[
    a:=(\alpha-1)\tau.
\]
First,
\[
    \mathbb E_{Q_u}\!\left[e^{(\alpha-1)T_{\theta^v}(X)}\right]
    =
    \sum_{i=1}^d \frac{1+\delta u_i}{d} e^{a v_i}.
\]
Because \(v\) is balanced,
\[
    \frac1d\sum_{i=1}^d e^{a v_i}
    =
    \frac{e^a+e^{-a}}{2}
    =
    \cosh(a).
\]
Also,
\[
    \frac1d\sum_{i=1}^d u_i e^{a v_i}
    =
    \frac1d\sum_{i=1}^d u_i\bigl(\cosh(a)+v_i\sinh(a)\bigr).
\]
Since \(u\) is balanced, \(\sum_i u_i=0\), so
\[
    \frac1d\sum_{i=1}^d u_i e^{a v_i}
    =
    \sinh(a)\frac1d\sum_{i=1}^d u_i v_i
    =
    \sinh(a)\rho(u,v).
\]
Therefore
\[
    \mathbb E_{Q_u}\!\left[e^{(\alpha-1)T_{\theta^v}(X)}\right]
    =
    \cosh(a)+\delta\sinh(a)\rho(u,v).
\]
Similarly, with
\[
    b:=\alpha\tau,
\]
we have
\[
    \mathbb E_P\!\left[e^{\alpha T_{\theta^v}(Y)}\right]
    =
    \frac1d\sum_{i=1}^d e^{b v_i}
    =
    \cosh(b),
\]
again because \(v\) is balanced. Hence
\[
    V_u(\theta^v)
    :=
    V_{P,Q_u}(\theta^v)
    =
    \frac{1}{\alpha-1}
    \log\!\left(\cosh(a)+\delta\sinh(a)\rho(u,v)\right)
    -
    \frac1\alpha\log\cosh(b).
\]
The second term is independent of both \(u\) and \(v\). Thus separation of the
DV objective is controlled by the first term.

\paragraph{Step 5: Separation of the correct critic.}
Define
\[
    F(r)
    :=
    \frac{1}{\alpha-1}
    \log\!\left(\cosh(a)+\delta\sinh(a)r\right),
    \qquad r\in[-1,1].
\]
For \(\delta\le 1/2\), the argument of the logarithm is strictly positive,
provided \(\tau\) is fixed. Differentiating gives
\[
    F'(r)
    =
    \frac{\delta\sinh(a)}
    {(\alpha-1)\left(\cosh(a)+\delta\sinh(a)r\right)}.
\]
Since \(a=(\alpha-1)\tau\), the quantity
\[
    \frac{\sinh(a)}{\alpha-1}
\]
is positive for every \(\alpha\ne 1\). Therefore \(F\) is increasing in \(r\).
Moreover, because \(\tau\) is a fixed constant depending only on \(\alpha\), and
\(\delta\le 1/2\), there exists a constant \(c_\alpha>0\), depending only on
\(\alpha\) and \(\tau\), such that
\[
    F'(r)\ge c_\alpha \delta
    \qquad\text{for all }r\in[-1,1].
\]
Indeed, the denominator is bounded above and below by positive constants
depending only on \(\alpha\) and \(\tau\).

For the true index \(u\), we have \(\rho(u,u)=1\). For any distinct
\(v\in\mathcal U_d\), we have \(\rho(u,v)\le 1/2\). Therefore
\[
    V_u(\theta^u)-V_u(\theta^v)
    =
    F(1)-F(\rho(u,v))
    \ge
    F(1)-F(1/2)
    \ge
    \frac{c_\alpha}{2}\delta.
\]
Let
\[
    \Delta:=\frac{c_\alpha}{2}\delta.
\]
Then, for every \(u\in\mathcal U_d\),
\[
    V_u(\theta^u)
    \ge
    \max_{v\ne u} V_u(\theta^v)+\Delta.
\]

\paragraph{Step 6: Uniform estimation implies decoding.}
Suppose an estimator \(\widehat V_n:\Theta\to\mathbb R\) satisfies
\[
    \sup_{\theta\in\Theta}
    |\widehat V_n(\theta)-V_u(\theta)|
    <
    \frac{\Delta}{4}.
\]
Define the decoder
\[
    \widehat u
    \in
    \arg\max_{v\in\mathcal U_d} \widehat V_n(\theta^v).
\]
Then \(\widehat u=u\). To see this, for every \(v\ne u\),
\[
    \widehat V_n(\theta^u)
    \ge
    V_u(\theta^u)-\frac{\Delta}{4}
    \ge
    V_u(\theta^v)+\Delta-\frac{\Delta}{4}
    =
    V_u(\theta^v)+\frac{3\Delta}{4}
    \ge
    \widehat V_n(\theta^v)+\frac{\Delta}{2}.
\]
Thus \(\theta^u\) is the unique maximizer of \(\widehat V_n\) over the packing
points. Consequently,
\[
    \Pr_u(\widehat u\ne u)
    \le
    \Pr_u\!\left[
        \sup_{\theta\in\Theta}
        |\widehat V_n(\theta)-V_u(\theta)|
        \ge
        \frac{\Delta}{4}
    \right].
\]

\paragraph{Step 7: KL control.}

We next bound the KL divergence between two hypotheses. For
\(u,v\in\mathcal U_d\),
\[
    D_{\mathrm{KL}}(Q_u\|Q_v)
    =
    \sum_{i=1}^d
    \frac{1+\delta u_i}{d}
    \log\frac{1+\delta u_i}{1+\delta v_i}.
\]
Since both \(u\) and \(v\) are balanced, the number of coordinates with
\((u_i,v_i)=(1,-1)\) equals the number of coordinates with
\((u_i,v_i)=(-1,1)\). Let this number be \(m\). Then
\(d_H(u,v)=2m\). Hence
\[
D_{\rm KL}(Q_u\|Q_v)
=
\frac{m}{d}(1+\delta)\log\frac{1+\delta}{1-\delta}
+
\frac{m}{d}(1-\delta)\log\frac{1-\delta}{1+\delta}.
\]
Therefore
\[
D_{\rm KL}(Q_u\|Q_v)
=
\frac{2m\delta}{d}\log\frac{1+\delta}{1-\delta}.
\]
For \(0<\delta\le 1/2\),
\[
\log\frac{1+\delta}{1-\delta}\le C\delta,
\]
so
\[
D_{\rm KL}(Q_u\|Q_v)
\le
C\delta^2\frac{m}{d}
\le C\delta^2.
\]


Applying this coordinatewise gives
\[
    D_{\mathrm{KL}}(Q_u\|Q_v)
    \le
    C\delta^2\frac{d_H(u,v)}{d}
    \le
    C\delta^2.
\]
The full observation law under \(u\) is
\[
    \mathbb P_u
    =
    Q_u^{\otimes n}\otimes P^{\otimes n}.
\]
Since \(P\) is the same under all hypotheses,
\[
    D_{\mathrm{KL}}(\mathbb P_u\|\mathbb P_v)
    =
    nD_{\mathrm{KL}}(Q_u\|Q_v)
    \le
    Cn\delta^2.
\]

\paragraph{Step 8: Fano's inequality.}
Let \(U\) be uniformly distributed on \(\mathcal U_d\), and let the samples be
drawn according to \(\mathbb P_U\). By the KL bound above,
\[
    I(U;\text{samples})
    \le
    \frac{1}{|\mathcal U_d|^2}
    \sum_{u,v\in\mathcal U_d}
    D_{\mathrm{KL}}(\mathbb P_u\|\mathbb P_v)
    \le
    Cn\delta^2.
\]
Fano's inequality~\citep{CoverT2006} gives, for every decoder \(\widehat U\),
\[
    \Pr(\widehat U\ne U)
    \ge
    1-
    \frac{I(U;\text{samples})+\log 2}{\log|\mathcal U_d|}.
\]
Since \(\log|\mathcal U_d|\ge c_{\mathrm{VG}}d\), if
\[
    Cn\delta^2+\log 2
    \le
    \frac12 c_{\mathrm{VG}}d,
\]
then
\[
    \Pr(\widehat U\ne U)\ge \frac12.
\]
In particular, there exists a universal constant \(c'>0\) such that if
\[
    n\le c'\frac{d}{\delta^2},
\]
then every decoder has error probability at least \(1/2\).

\paragraph{Step 9: Convert testing hardness into estimation hardness.}
Set
\[
    \delta:=\frac{8\varepsilon}{c_\alpha}.
\]
For \(\varepsilon\le c_1\), with \(c_1>0\) sufficiently small depending only on
\(\alpha\), this choice satisfies \(\delta\le 1/2\). With this choice,
\[
    \frac{\Delta}{4}
    =
    \frac{1}{4}\cdot \frac{c_\alpha}{2}\delta
    =
    \frac{c_\alpha\delta}{8}
    =
    \varepsilon.
\]
Suppose, for contradiction, that there exists an estimator \(\widehat V_n\) such
that
\[
    \sup_{u\in\mathcal U_d}
    \Pr_u\!\left[
        \sup_{\theta\in\Theta}
        |\widehat V_n(\theta)-V_u(\theta)|
        \ge \varepsilon
    \right]
    <
    \frac14.
\]
Then the decoder
\[
    \widehat u
    \in
    \arg\max_{v\in\mathcal U_d} \widehat V_n(\theta^v)
\]
would satisfy
\[
    \sup_{u\in\mathcal U_d}
    \Pr_u(\widehat u\ne u)
    <
    \frac14,
\]
contradicting Fano's inequality whenever
\[
    n\le c'\frac{d}{\delta^2}
    =
    c'\frac{c_\alpha^2}{64}\frac{d}{\varepsilon^2}.
\]
Absorbing constants into \(c>0\), we obtain
\[
    n\le c\frac{d}{\varepsilon^2}
    \quad\Longrightarrow\quad
    \inf_{\widehat V_n}
    \sup_{u\in\mathcal U_d}
    \Pr_u\!\left[
        \sup_{\theta\in\Theta}
        |\widehat V_n(\theta)-V_u(\theta)|
        \ge \varepsilon
    \right]
    \ge
    \frac14.
\]
This is the desired lower bound.
\end{proof}

\begin{remark}[Tightness of Theorem~\ref{thm:finite-sample-dv-renyi-app}]
Theorem~\ref{thm:finite-sample-dv-renyi-app} shows that the leading $d/\varepsilon^2$
sample complexity in Theorem~\ref{thm:finite-sample-dv-renyi} is minimax-optimal for class-restricted DV R\'enyi estimation,
up to logarithmic factors coming from covering/metric entropy.
\end{remark}

\subsection{Group Privacy and Multiple Canaries}
One can have multiple canaries in our auditing set up such that the datasets we input into the mechanism have a distance greater than  one. In this case, we follow the group privacy result from \cite{mironov2017renyi}:
\begin{proposition}[Proposition 2 in \cite{mironov2017renyi}]
If $f : \mathcal{D} \to \mathcal{R}$ is $(\alpha,\varepsilon)$-RDP; 
$g : \mathcal{D}' \to \mathcal{D}$ is $2^c$-stable and $\alpha \ge 2^{c+1}$, 
then $f \circ g$ is $(\alpha / 2^c, 3^c \varepsilon)$-RDP.
\end{proposition}

\begin{lemma}
\label{lem:conversion-app}
For each $M$ and a fixed $\alpha > 1$, let $A_M \in \Omega$ be a random variable and let $P_M \in \mathbb{R}$ be a fixed number.
For each $\varepsilon,\beta>0$, let $T_{\varepsilon,\alpha,\beta}\subset \Omega$ satisfy
\begin{equation}
\forall M\qquad
\Bigl(P_M=\varepsilon \;\Longrightarrow\; \mathbb{P}\bigl[A_M \in T_{\varepsilon,\alpha,\beta}\bigr]\le \beta\Bigr).
\end{equation}
Further suppose that, if $\varepsilon_1 \le \varepsilon_2$, then
$T_{\varepsilon_1,\alpha,\beta} \supset T_{\varepsilon_2,\alpha,\beta}$.
Then, for all $M$ and all $\beta>0$,
\begin{equation}
\Pr\!\left\{\,P_M \ge
\sup\{\varepsilon>0 : A_M \in T_{\varepsilon,\alpha,\beta}\}\right\}
\ge 1-\beta.
\end{equation}
\end{lemma}
\begin{proof}
We follow a very similar proof as outlined in Lemma 4.7 of \cite{steinke2023privacy}, just using the monotonicity of RDP and subsequently the null hypothesis's rejection region. Fix an output of $A_M$ and let $P_M < \sup\{\varepsilon>0 : A_M \in T_{\varepsilon,\alpha,\beta}\}$. Then there is a $\varepsilon \ge P_M$ with $A_M \in T_{\varepsilon,\alpha,\beta}$, and therefore 
\begin{align}
    \label{eqn:union}
    A_M \in \bigcup_{\varepsilon \ge P_M} T_{\varepsilon,\alpha,\beta} = T_{P_M,\alpha,\beta}.
\end{align}
By monotonicity of the rejection region. Thus, 
\begin{align}
\Pr\!\left[P_M < \sup\{\varepsilon > 0 : A_M \in T_{\varepsilon,\alpha, \beta}\}\right]
\le
\Pr\bigl[A_M \in T_{P_M,\alpha,\beta}\bigr]
\le \beta.
\end{align}

For the above to hold, we assume that $\alpha > 1$ is fixed. 
\end{proof}

\section{Additional Experimental Results}
\label{sec:more-experiments}

\begin{figure*}
    \centering
    \includegraphics[width=1.0\linewidth]{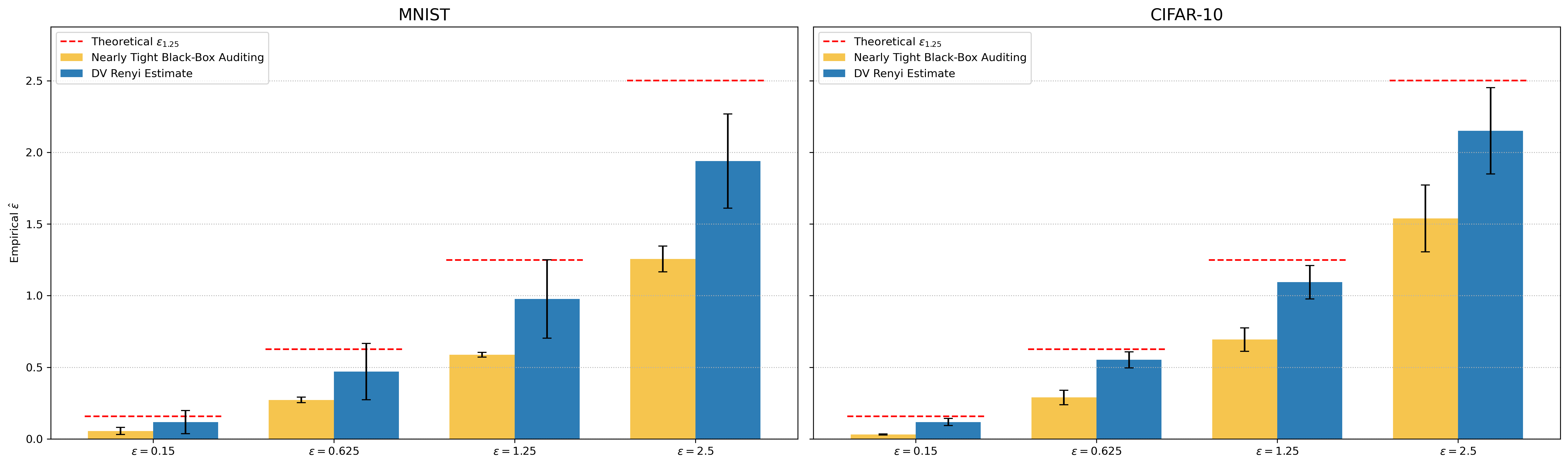}
    \caption{Graph for auditing at $\alpha = 1.25$ for CIFAR-10 and MNIST datasets on a CNN}
    \label{fig:125worst_case}
\end{figure*}

\begin{figure*}
    \centering
    \includegraphics[width=1.0\linewidth]{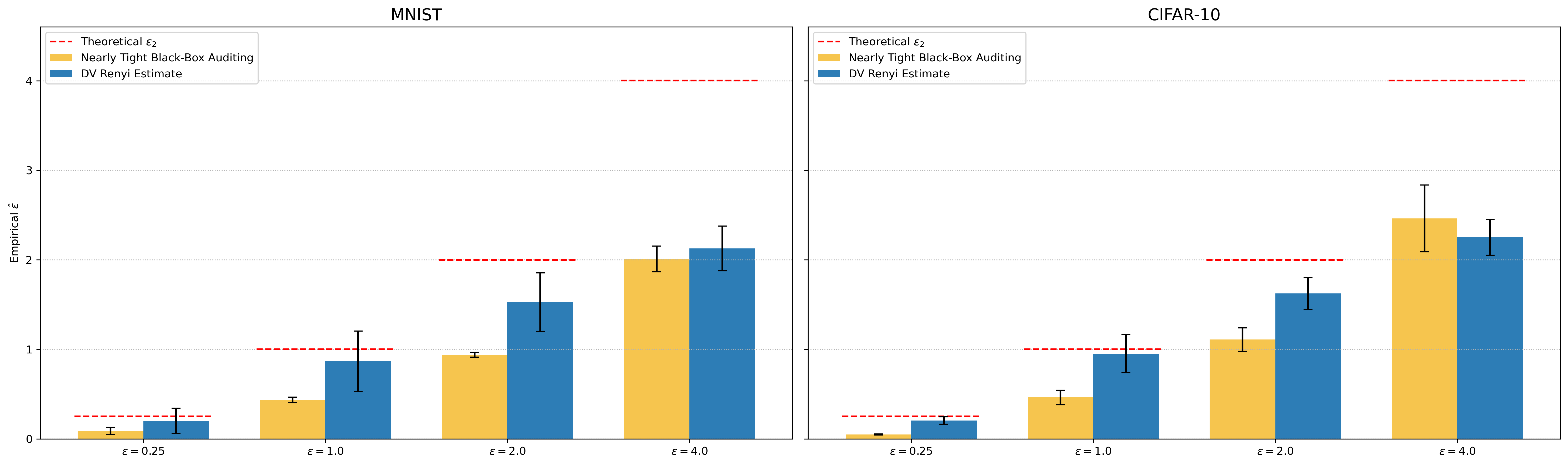}
    \caption{Graph for auditing at $\alpha = 2.0$ for CIFAR-10 and MNIST datasets on a CNN}
    \label{fig:2worst_case}
\end{figure*}

\begin{table}[t]
\centering
\caption{Empirical RDP audits at $\alpha=1.25$. Same as in Section~\ref{sec:experiments}}
\label{tab:alpha125-1.0}
\begin{tabular}{llcc}
\toprule
Dataset & Target $\varepsilon_\alpha$ & SOTA black-box \cite{muthu2024nearly} & DV-R\'enyi \\
\midrule
\multirow{5}{*}{MNIST} & 0.15625 & $0.056 \pm 0.025$ & $\mathbf{0.115} \pm 0.080$ \\
 & 0.625 & $0.272 \pm 0.020$ & $\mathbf{0.469} \pm 0.197$ \\
 & 1.25 & $0.588 \pm 0.016$ & $\mathbf{0.995} \pm 0.289$ \\
 & 2.5 & $1.256 \pm 0.091$ & $\mathbf{2.239} \pm 0.473$ \\
 & 6.25 & $3.325 \pm 0.218$ & $\mathbf{3.884} \pm 0.327$ \\
\midrule
\multirow{5}{*}{CIFAR-10} & 0.15625 & $0.031 \pm 0.004$ & $\mathbf{0.118} \pm 0.025$ \\
 & 0.625 & $0.289 \pm 0.050$ & $\mathbf{0.551} \pm 0.056$ \\
 & 1.25 & $0.693 \pm 0.082$ & $\mathbf{1.095} \pm 0.120$ \\
 & 2.5 & $1.539 \pm 0.234$ & $\mathbf{2.248} \pm 0.344$ \\
 & 6.25 & $4.128 \pm 0.336$ & $\mathbf{4.577} \pm 0.430$ \\
\bottomrule
\end{tabular}
\end{table}

\begin{table}[t]
\centering
\caption{Empirical RDP audits at $\alpha=1.25$, with Dataset size: 100 samples.}
\label{tab:alpha125-100samples}
\begin{tabular}{llcc}
\toprule
Dataset & Target $\varepsilon_\alpha$ & SOTA black-box \cite{muthu2024nearly} & DV-R\'enyi \\
\midrule
\multirow{5}{*}{MNIST} & 0.15625 & $0.086 \pm 0.042$ & $\mathbf{0.152} \pm 0.061$ \\
 & 0.625 & $\mathbf{0.506} \pm 0.124$ & $0.184 \pm 0.029$ \\
 & 1.25 & $\mathbf{0.875} \pm 0.075$ & $0.172 \pm 0.014$ \\
 & 2.5 & $\mathbf{1.921} \pm 0.249$ & $0.176 \pm 0.011$ \\
 & 6.25 & $\mathbf{4.685} \pm 0.370$ & $0.177 \pm 0.006$ \\
\midrule
\multirow{5}{*}{CIFAR-10} & 0.15625 & $0.036 \pm 0.010$ & $\mathbf{0.147} \pm 0.034$ \\
 & 0.625 & $0.323 \pm 0.039$ & $\mathbf{0.486} \pm 0.053$ \\
 & 1.25 & $\mathbf{0.730} \pm 0.085$ & $0.649 \pm 0.052$ \\
 & 2.5 & $\mathbf{1.631} \pm 0.179$ & $0.767 \pm 0.045$ \\
 & 6.25 & $\mathbf{4.341} \pm 0.368$ & $0.848 \pm 0.045$ \\
\bottomrule
\end{tabular}
\end{table}

\begin{table}[t]
\centering
\caption{Empirical RDP audits at $\alpha=1.25$, with Dataset size: 1000 samples.}
\label{tab:alpha125-1000samples}
\begin{tabular}{llcc}
\toprule
Dataset & Target $\varepsilon_\alpha$ & SOTA black-box \cite{muthu2024nearly} & DV-R\'enyi \\
\midrule
\multirow{5}{*}{MNIST} & 0.15625 & $0.095 \pm 0.057$ & $\mathbf{0.155} \pm 0.068$ \\
 & 0.625 & $\mathbf{0.493} \pm 0.104$ & $0.185 \pm 0.034$ \\
 & 1.25 & $\mathbf{0.863} \pm 0.104$ & $0.167 \pm 0.015$ \\
 & 2.5 & $\mathbf{1.844} \pm 0.154$ & $0.172 \pm 0.009$ \\
 & 6.25 & $\mathbf{4.478} \pm 0.310$ & $0.174 \pm 0.007$ \\
\midrule
\multirow{5}{*}{CIFAR-10} & 0.15625 & $0.031 \pm 0.004$ & $\mathbf{0.118} \pm 0.025$ \\
 & 0.625 & $0.289 \pm 0.050$ & $\mathbf{0.551} \pm 0.056$ \\
 & 1.25 & $0.693 \pm 0.082$ & $\mathbf{1.095} \pm 0.120$ \\
 & 2.5 & $1.539 \pm 0.234$ & $\mathbf{2.248} \pm 0.344$ \\
 & 6.25 & $4.128 \pm 0.336$ & $\mathbf{4.577} \pm 0.430$ \\
\bottomrule
\end{tabular}
\end{table}

\begin{table}[t]
\centering
\caption{Empirical RDP audits at $\alpha=1.25$, with Max grad norm: 0.1.}
\label{tab:alpha125-0.1}
\begin{tabular}{llcc}
\toprule
Dataset & Target $\varepsilon_\alpha$ & SOTA black-box \cite{muthu2024nearly} & DV-R\'enyi \\
\midrule
\multirow{5}{*}{MNIST} & 0.15625 & $0.084 \pm 0.044$ & $\mathbf{0.147} \pm 0.084$ \\
 & 0.625 & $0.340 \pm 0.058$ & $\mathbf{0.394} \pm 0.074$ \\
 & 1.25 & $\mathbf{0.729} \pm 0.071$ & $0.475 \pm 0.052$ \\
 & 2.5 & $\mathbf{1.592} \pm 0.139$ & $0.512 \pm 0.036$ \\
 & 6.25 & $\mathbf{4.076} \pm 0.357$ & $0.524 \pm 0.025$ \\
\midrule
\multirow{5}{*}{CIFAR-10} & 0.15625 & $0.046 \pm 0.031$ & $\mathbf{0.138} \pm 0.023$ \\
 & 0.625 & $0.301 \pm 0.032$ & $\mathbf{0.556} \pm 0.062$ \\
 & 1.25 & $0.758 \pm 0.082$ & $\mathbf{1.018} \pm 0.130$ \\
 & 2.5 & $\mathbf{1.628} \pm 0.126$ & $1.369 \pm 0.142$ \\
 & 6.25 & $\mathbf{4.342} \pm 0.329$ & $1.692 \pm 0.059$ \\
\bottomrule
\end{tabular}
\end{table}

\begin{table}[t]
\centering
\caption{Empirical RDP audits at $\alpha=1.25$, with Max grad norm: 10.0.}
\label{tab:alpha125-10.0}
\begin{tabular}{llcc}
\toprule
Dataset & Target $\varepsilon_\alpha$ & SOTA black-box \cite{muthu2024nearly} & DV-R\'enyi \\
\midrule
\multirow{5}{*}{MNIST} & 0.15625 & $0.001 \pm 0.001$ & $\mathbf{0.006} \pm 0.007$ \\
 & 0.625 & $0.033 \pm 0.026$ & $\mathbf{0.098} \pm 0.112$ \\
 & 1.25 & $0.035 \pm 0.022$ & $\mathbf{0.070} \pm 0.086$ \\
 & 2.5 & $0.035 \pm 0.016$ & $\mathbf{0.118} \pm 0.074$ \\
 & 6.25 & $0.059 \pm 0.028$ & $\mathbf{4.764} \pm 9.191$ \\
\midrule
\multirow{5}{*}{CIFAR-10} & 0.15625 & $0.002 \pm 0.001$ & $\mathbf{0.005} \pm 0.009$ \\
 & 0.625 & $0.065 \pm 0.030$ & $\mathbf{0.166} \pm 0.092$ \\
 & 1.25 & $0.191 \pm 0.012$ & $\mathbf{0.456} \pm 0.118$ \\
 & 2.5 & $0.652 \pm 0.035$ & $\mathbf{1.657} \pm 0.275$ \\
 & 6.25 & $2.445 \pm 0.130$ & $\mathbf{5.577} \pm 2.545$ \\
\bottomrule
\end{tabular}
\end{table}


\begin{table}[t]
\centering
\caption{Empirical RDP audits at $\alpha=2$, same as in Section~\ref{sec:experiments}.}
\label{tab:alpha2-1.0}
\begin{tabular}{llcc}
\toprule
Dataset & Target $\varepsilon_\alpha$ & SOTA black-box \cite{muthu2024nearly} & DV-R\'enyi \\
\midrule
\multirow{5}{*}{MNIST} & 0.25 & $0.089 \pm 0.040$ & $\mathbf{0.202} \pm 0.142$ \\
 & 1.0 & $0.436 \pm 0.032$ & $\mathbf{0.867} \pm 0.339$ \\
 & 2.0 & $0.940 \pm 0.025$ & $\mathbf{1.527} \pm 0.325$ \\
 & 4.0 & $2.010 \pm 0.145$ & $\mathbf{2.127} \pm 0.250$ \\
 & 10.0 & $\mathbf{5.320} \pm 0.349$ & $2.650 \pm 0.159$ \\
\midrule
\multirow{5}{*}{CIFAR-10} & 0.25 & $0.050 \pm 0.007$ & $\mathbf{0.206} \pm 0.041$ \\
 & 1.0 & $0.463 \pm 0.080$ & $\mathbf{0.953} \pm 0.213$ \\
 & 2.0 & $1.109 \pm 0.131$ & $\mathbf{1.624} \pm 0.178$ \\
 & 4.0 & $\mathbf{2.463} \pm 0.374$ & $2.250 \pm 0.200$ \\
 & 10.0 & $\mathbf{6.605} \pm 0.538$ & $2.751 \pm 0.170$ \\
\bottomrule
\end{tabular}
\end{table}

\begin{table}[t]
\centering
\caption{Empirical RDP audits at $\alpha=2$, with Dataset size: 100 samples.}
\label{tab:alpha2-100samples}
\begin{tabular}{llcc}
\toprule
Dataset & Target $\varepsilon_\alpha$ & SOTA black-box \cite{muthu2024nearly} & DV-R\'enyi \\
\midrule
\multirow{5}{*}{MNIST} & 0.25 & $0.138 \pm 0.067$ & $\mathbf{0.407} \pm 0.335$ \\
 & 1.0 & $0.810 \pm 0.198$ & $\mathbf{1.512} \pm 0.662$ \\
 & 2.0 & $1.399 \pm 0.120$ & $\mathbf{1.910} \pm 0.390$ \\
 & 4.0 & $\mathbf{3.074} \pm 0.398$ & $2.360 \pm 0.294$ \\
 & 10.0 & $\mathbf{7.496} \pm 0.592$ & $2.796 \pm 0.150$ \\
\midrule
\multirow{5}{*}{CIFAR-10} & 0.25 & $0.058 \pm 0.015$ & $\mathbf{0.257} \pm 0.048$ \\
 & 1.0 & $0.517 \pm 0.063$ & $\mathbf{1.138} \pm 0.128$ \\
 & 2.0 & $1.169 \pm 0.136$ & $\mathbf{1.829} \pm 0.098$ \\
 & 4.0 & $\mathbf{2.609} \pm 0.286$ & $2.322 \pm 0.146$ \\
 & 10.0 & $\mathbf{6.945} \pm 0.589$ & $2.845 \pm 0.185$ \\
\bottomrule
\end{tabular}
\end{table}

\begin{table}[t]
\centering
\caption{Empirical RDP audits at $\alpha=2$, with Dataset size: 1000 samples.}
\label{tab:alpha2-1000samples}
\begin{tabular}{llcc}
\toprule
Dataset & Target $\varepsilon_\alpha$ & SOTA black-box \cite{muthu2024nearly} & DV-R\'enyi \\
\midrule
\multirow{5}{*}{MNIST} & 0.25 & $0.151 \pm 0.091$ & $\mathbf{0.381} \pm 0.262$ \\
 & 1.0 & $0.789 \pm 0.166$ & $\mathbf{1.272} \pm 0.427$ \\
 & 2.0 & $1.380 \pm 0.166$ & $\mathbf{1.822} \pm 0.338$ \\
 & 4.0 & $\mathbf{2.950} \pm 0.246$ & $2.218 \pm 0.242$ \\
 & 10.0 & $\mathbf{7.165} \pm 0.495$ & $2.618 \pm 0.171$ \\
\midrule
\multirow{5}{*}{CIFAR-10} & 0.25 & $0.050 \pm 0.007$ & $\mathbf{0.206} \pm 0.041$ \\
 & 1.0 & $0.463 \pm 0.080$ & $\mathbf{0.953} \pm 0.213$ \\
 & 2.0 & $1.109 \pm 0.131$ & $\mathbf{1.624} \pm 0.178$ \\
 & 4.0 & $\mathbf{2.463} \pm 0.374$ & $2.250 \pm 0.200$ \\
 & 10.0 & $\mathbf{6.605} \pm 0.538$ & $2.751 \pm 0.170$ \\
\bottomrule
\end{tabular}
\end{table}

\begin{table}[t]
\centering
\caption{Empirical RDP audits at $\alpha=2$, with Max grad norm: 0.1.}
\label{tab:alpha2-0.1}
\begin{tabular}{llcc}
\toprule
Dataset & Target $\varepsilon_\alpha$ & SOTA black-box \cite{muthu2024nearly} & DV-R\'enyi \\
\midrule
\multirow{5}{*}{MNIST} & 0.25 & $0.134 \pm 0.070$ & $\mathbf{0.175} \pm 0.070$ \\
 & 1.0 & $\mathbf{0.543} \pm 0.094$ & $0.265 \pm 0.038$ \\
 & 2.0 & $\mathbf{1.166} \pm 0.114$ & $0.284 \pm 0.028$ \\
 & 4.0 & $\mathbf{2.548} \pm 0.222$ & $0.292 \pm 0.022$ \\
 & 10.0 & $\mathbf{6.522} \pm 0.571$ & $0.296 \pm 0.012$ \\
\midrule
\multirow{5}{*}{CIFAR-10} & 0.25 & $0.073 \pm 0.050$ & $\mathbf{0.186} \pm 0.028$ \\
 & 1.0 & $\mathbf{0.481} \pm 0.052$ & $0.304 \pm 0.017$ \\
 & 2.0 & $\mathbf{1.212} \pm 0.131$ & $0.329 \pm 0.011$ \\
 & 4.0 & $\mathbf{2.604} \pm 0.202$ & $0.342 \pm 0.007$ \\
 & 10.0 & $\mathbf{6.947} \pm 0.527$ & $0.349 \pm 0.002$ \\
\bottomrule
\end{tabular}
\end{table}

\begin{table}[t]
\centering
\caption{Empirical RDP audits at $\alpha=2$, with Max grad norm: 10.0.}
\label{tab:alpha2-10.0}
\begin{tabular}{llcc}
\toprule
Dataset & Target $\varepsilon_\alpha$ & SOTA black-box \cite{muthu2024nearly} & DV-R\'enyi \\
\midrule
\multirow{5}{*}{MNIST} & 0.25 & $0.001 \pm 0.002$ & $\mathbf{0.014} \pm 0.011$ \\
 & 1.0 & $0.052 \pm 0.042$ & $\mathbf{0.164} \pm 0.184$ \\
 & 2.0 & $0.056 \pm 0.035$ & $\mathbf{0.131} \pm 0.148$ \\
 & 4.0 & $0.056 \pm 0.025$ & $\mathbf{0.210} \pm 0.175$  \\
\midrule
\multirow{5}{*}{CIFAR-10} & 0.25 & $0.004 \pm 0.002$ & $\mathbf{0.012} \pm 0.019$ \\
 & 1.0 & $0.104 \pm 0.049$ & $\mathbf{0.324} \pm 0.192$ \\
 & 2.0 & $0.306 \pm 0.019$ & $\mathbf{2.200} \pm 2.979$ \\
 & 4.0 & $1.044 \pm 0.056$ & $\mathbf{6.983} \pm 3.233$ \\
\bottomrule
\end{tabular}
\end{table}

This section in the appendix
provides additional interpretation of the experimental results reported in Section~\ref{sec:experiments} (in main body) 
and the accompanying figures.
The primary purpose of these experiments is not to propose a new attack on DP-SGD, but rather to validate the theoretical
auditing guarantees developed in Section~\ref{sec:theory} and~\ref{sec:experiments} under realistic machine learning workloads.
In particular, the experiments are designed to answer the following questions:

\begin{itemize}
    \item Do DV-based R\'enyi divergence estimators produce empirically meaningful lower bounds on privacy loss~\citep{Alabi26}?
    \item How closely do empirical audits track the theoretically claimed RDP guarantees?
    \item How does estimator choice affect auditing accuracy, especially at small privacy budgets?
\end{itemize}

Across all datasets and models, the experimental findings are consistent with the predictions of our theory. It took around 100 GPU hours on A100 to audit an model trained on MNIST data and 500 GPU hours for a model trained on CIFAR-10 data.

\paragraph{Worst-case initializations.}
A recurring theme in prior work is that DP-SGD privacy guarantees must hold even for adversarially chosen initial parameters.
To stress-test our auditing framework, we therefore follow the experimental protocol of
\cite{AD24, muthu2024nearly} and evaluate models initialized using worst-case pretraining strategies. All of our experiments occur in this setting.
Despite the adversarial initialization, the DV-based R\'enyi estimator consistently produces nontrivial lower bounds
that closely track the claimed theoretical privacy levels.
This behavior is particularly notable at smaller values of~$\varepsilon$, where auditing is statistically most challenging.
These results empirically corroborate the distribution-free lower confidence bounds established in our theory
and demonstrate that our estimator remains stable even under worst-case conditions.

\paragraph{Comparison of estimators.}
We compare the following approaches to estimating R\'enyi divergence from observed losses:
(i) our DV-based variational estimator,
(ii) nearly tight black-box auditing results from prior work.

Across all experiments, the DV-based estimator achieves a favorable balance between bias and variance.
While prior black-box methods can be overly conservative at small $\varepsilon$,
the DV estimator yields tighter empirical lower bounds that remain statistically valid for low $\varepsilon$.
This behavior is especially pronounced in the low-$\varepsilon$ regime, where small estimation errors
translate into large relative differences in privacy guarantees.

These trends are consistent with our finite-sample analysis,
which predicts that the DV estimator can achieve near-optimal accuracy given a sufficiently expressive critic class.

\subsection{DV R\'enyi Model}
In our auditing procedure we use a DV R\'enyi model to calculate a divergence estimate that serves as an empirical RDP estimate. For this model, consider a dataset (consisting of final model losses trained with DP-SGD) with size $n$. Our neural networks use two intermediate layers of 100 nodes.

At $\alpha = 1.25$, we use 750 epochs and lr = 7e-5 for CIFAR-10 and 500 epochs with lr = 2e-4 for MNIST, and a batch size of 400. We use an EMA rate of 0.99 and our neural network is structured with two intermediate layers of 100 nodes with a ReLU activation function. 400 (80$\%$) of the loss samples are used in the training set and the last 100 (20$\%$) of the loss samples are used in the validation set. We scale our DV estimate $R_\alpha$ by $\alpha$ so $D_\alpha = \alpha R_\alpha$. At $\alpha = 2$, all the parameters are the same besides a slight decrease of the iterative learning, setting lr = 1.25e-4 for MNIST and 350 training epochs for CIFAR-10. As $\alpha$ increases, the model becomes increasingly sensitive and overall learning of the model needs to be decreased to avoid high MSE due to the divergence's exponential nature. To apply our theorems to our experimental settings and results, consistency results for our neural networks satisfying Theorem \ref{thm:finite-sample-dv-renyi}, Theorem~\ref{thm:dv-renyi-lower-bound-app} can be found in \cite{birrell2021variational}.


\subsection{Tightness lost in GDP to RDP conversions}

A concern when comparing RDP audits to prior black-box auditing methods is that converting an audit through intermediate privacy notions can introduce unnecessary slack. In particular, \cite{muthu2024nearly} produces a binary classifier that distinguishes whether a model was trained on dataset $D$ or its neighboring dataset $D'$. This classifier is summarized by its true-positive rate $\mathrm{TPR}$ and false-positive rate $\mathrm{FPR}$, which can then be used to estimate a $\mu$-GDP lower bound. However, converting this estimate to GDP and then to RDP may lose tightness. Therefore, we evaluate a direct RDP baseline from the binary attack itself. The binary distinguisher induces two Bernoulli distributions: one corresponding to the attack output under $D$, with success probability $\mathrm{TPR}$, and one corresponding to the attack output under $D'$, with success probability $\mathrm{FPR}$. Therefore, instead of passing through an intermediate GDP estimate, we directly compute the order-$\alpha$ Rényi divergence
\[
\hat{\varepsilon}
=
D_{\alpha}
\left(
\mathrm{Bern}(\mathrm{TPR})
\,\middle\|\,
\mathrm{Bern}(\mathrm{FPR})
\right).
\]
Although this provides a more direct RDP lower bound, the Bernoulli distributions from $\mathrm{TPR}$ and $\mathrm{FPR}$ are not robust to R\'enyi divergence estimation, shown in Table~\ref{tab:rdp-direct-comparison}.

\begin{table}[t]
\centering
\caption{Comparison of empirical RDP lower bounds at $\alpha=2$, with a direct $\hat{\varepsilon}=D_{\alpha}(\mathrm{Bern(TPR)}\Vert \mathrm{Bern(FPR)})$ estimate for previous SOTA.}
\label{tab:rdp-direct-comparison}
\begin{tabular}{lccc}
\toprule
Dataset &
Theor. $\varepsilon$ at $\alpha=2$ &
$\hat{\varepsilon}=D_{\alpha}(\mathrm{Bern(TPR)}\Vert \mathrm{Bern(FPR)})$ &
$\hat{\varepsilon}$-RDP DV \\
\midrule
\textbf{CIFAR-10} & 0.25 & $0.097 \pm 0.059$ & $\mathbf{0.206} \pm 0.041$ \\
 & 1.0 & $0.375 \pm 0.11$ & $\mathbf{0.953} \pm 0.213$ \\
 & 2.0 & $0.685 \pm 0.124$ & $\mathbf{1.624} \pm 0.178$ \\
 & 4.0 & $1.097 \pm 0.134$ & $\mathbf{2.25} \pm 0.2$ \\
 & 10.0 & $1.446 \pm 0.239$ & $\mathbf{2.751} \pm 0.17$ \\
\midrule
\textbf{MNIST} & 0.25 & $0.115 \pm 0.064$ & $\mathbf{0.202} \pm 0.142$ \\
 & 1.0 & $0.35 \pm 0.115$ & $\mathbf{0.867} \pm 0.339$ \\
 & 2.0 & $0.618 \pm 0.135$ & $\mathbf{1.527} \pm 0.325$ \\
 & 4.0 & $1.022 \pm 0.14$ & $\mathbf{2.127} \pm 0.25$ \\
 & 10.0 & $1.444 \pm 0.186$ & $\mathbf{2.65} \pm 0.159$ \\
\bottomrule
\end{tabular}
\end{table}